\documentclass[a4paper]{article}
\usepackage{amsmath}
  \usepackage{paralist}
  \usepackage{graphicx}
  \usepackage{amssymb,amsthm}
  \usepackage{authblk}
  \usepackage{xcolor}
  \usepackage{url}
  \usepackage[pagewise]{lineno} %\linenumbers
  \usepackage[colorlinks=true]{hyperref}
  \usepackage{algorithm}
  \usepackage{algpseudocode}
\hypersetup{urlcolor=blue, citecolor=red}

\newtheorem{theorem}{Theorem}[section]
\newtheorem{corollary}[theorem]{Corollary}

\newtheorem{lemma}[theorem]{Lemma}

\theoremstyle{definition}
\newtheorem{definition}[theorem]{Definition}
\newtheorem{remark}[theorem]{Remark}

\newcommand{\bfx}{{\bf x}}
\newcommand{\bfy}{{\bf y}}

\begin{document}

\title{Power Weighted Shortest Paths for Clustering Euclidean Data}
\author[1]{Daniel Mckenzie \thanks{Corresponding Author: mckenzie@math.ucla.edu}}
\author[2]{Steven Damelin \thanks{damelin@umich.edu}}
\affil[1]{Department of Mathematics, University of California, Los Angeles}
\affil[2]{Department of Mathematics, University of Michigan}

\maketitle

%The abstract of your paper
\begin{abstract}
We study the use of power weighted shortest path metrics for clustering high dimensional Euclidean data, under the assumption that the data is drawn from a collection of disjoint low dimensional manifolds. We argue, theoretically and experimentally, that this leads to higher clustering accuracy. We also present a fast algorithm for computing these distances.
\end{abstract}


\section{Introduction}
Clustering high dimensional data is an increasingly important problem in contemporary unsupervised machine learning. Here, we shall consider this problem in the case where our data is presented as a subset of a Euclidean space, $\mathcal{X}\subset \mathbb{R}^{D}$, although our results easily extend to more general metric spaces. Loosely speaking, by clustering we mean partitioning $\mathcal{X}$ into $\ell$ subsets, or clusters, $\mathcal{X} = \mathcal{X}_{1}\cup\ldots\cup\mathcal{X}_{\ell}$ such that data points in the same $\mathcal{X}_{a}$ are more ``similar'' than data points in different subsets. Clearly, the notion of similarity is context dependent. Although there exist algorithms that operate on the data directly, for example $k$-means, many modern algorithms proceed by first representing the data as a weighted graph $G = (V,E,A)$ with $V =\{1,\ldots, n\}$ and $A_{ij}$ representing the similarity between $\bfx_i$ and $\bfx_j$ and then using a graph clustering algorithm on $G$. Spectral clustering \cite{Ng2002} is an archetypal example of such an approach. Constructing $A$ requires a choice of metric $d(\cdot,\cdot): \mathbb{R}^{D}\times\mathbb{R}^{D}\to \mathbb{R}$. Ideally, one should choose $d$ such that points in the same cluster are close with respect to $d(\cdot,\cdot)$, while points in different clusters remain distant. Thus, the choice of metric should reflect, in some way, our assumptions about the data $\mathcal{X}$ and the notion of similarity we would like the clusters to reflect. \\ 

A common assumption, frequently referred to as the {\em manifold hypothesis} (see, for example, \cite{Fefferman2016}) posits that each $\mathcal{X}_a$ is sampled from a latent data manifold $\mathcal{M}_a$. Many types of data sets are known or suspected to satisfy this hypothesis, for example motion segmentation \cite{Costeira1998,Aldroubi2019}, images of faces or objects taken from different angles or under different lighting \cite{Basri2003,Ho2003} or handwritten digits \cite{Tenenbaum2000}. It is also usually assumed that the dimension of each $\mathcal{M}_a$ is much lower than the ambient dimension $D$. Although it can be shown that taking $d(\cdot,\cdot)$ to be the Euclidean distance can be successful \cite{Arias2011} for such data, {\em data-driven} metrics have been increasingly favored \cite{Coifman2006,Bijral2011,Chu2017,Little2017}. \\

\begin{figure}
    \centering
    \includegraphics[width=0.5\linewidth]{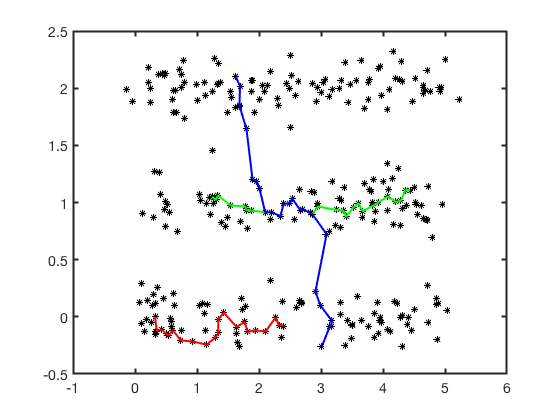}
    \caption{Three sample geodesics in the power weighted shortest path metric with $p=2$, for the data set ``Three Lines'' (see \S \ref{sec:ExperimentalResults}). Observe how the geodesics consist of many small hops, instead of several large hops. The total lengths of the red and green paths are significantly smaller than the length of the blue path.}
    \label{fig:Sample_Geodesics}
\end{figure}

Once $d(\cdot,\cdot)$ has been chosen, $A$ can be constructed. A common choice \cite{Ng2002,Zelnik2005} is to use some variant of a Gaussian kernel, whereby $A_{ij} = \exp(-d^2(\bfx_i,\bfx_j)/\sigma^{2})$ for a user defined parameter $\sigma$. However this is unsuitable for large data sets, as the resulting similarity matrix is dense and thus may be too large to store in memory. Moreover, operations which form core parts of many clustering algorithms, for example computing eigenvectors, are prohibitively expensive when $A$ is full. Hence in this case, a common choice that yields a sparse similarity matrix is to use a $k$ nearest neighbors ($k$-NN) graph constructed as:
$$
A_{ij} = \left\{\begin{array}{cc} 1 & \text{ if $\bfx_i$ among the $k$ nearest neighbors of $\bfx_j$ with respect to $d(\cdot,\cdot)$} \\ 0 & \text{ otherwise }\end{array}\right.
$$
Here $k$ is a user specified parameter. \\

In this article we consider taking $d(\cdot,\cdot)$ to be a  {\em power weighted shortest path metric} (henceforth: $p$-wspm and defined in \S \ref{section:PowerWeightedSPM}) with an emphasis on cases where the data satisfies the manifold hypothesis and where the data set is so large that a $k$-NN similarity matrix is preferable to a full similarity matrix. The use of shortest path metrics in clustering data is not new (see the discussion in \S \ref{sec:PriorWork}), but has typically been hindered by high computational cost. Indeed finding the pairwise distance between all $\bfx_{\alpha},\bfx_{\beta}\in\mathcal{X}$ in the shortest path metric is equivalent to the all pairs shortest paths problem on a complete weighted graph, which requires $O(n^3)$ operations using the Floyd-Warshall algorithm. We provide a way around this computational barrier, and also contribute to the theoretical analysis of $p$-wspm's. Specifically, our contributions are:
 
\begin{enumerate}
    \item We prove that $p$-wspm's behave as expected for data satisfying the manifold hypothesis. That is, we show that the maximum distance between points in the same cluster is small with high probability, and tends to zero as the number of data points tends to infinity. On the other hand, the maximum distance between points in different clusters remains bounded away from zero.
    
    \item We show how $p$-wspm's can be thought of as interpolants between the Euclidean metric and the {\em longest leg path distance} (defined in \S \ref{sec:LLPD}), which we shall abbreviate to LLPD.
    \item We introduce a novel modified version of Dijkstra's algorithm that computes the $k$ nearest neighbors, with respect to any $p$-wspm or the LLPD, of any $\bfx_{\alpha}$ in $\mathcal{X}$ in $O(k^{2}\mathcal{T}_{Enn})$ time, where $\mathcal{T}_{Enn}$ is the cost of a Euclidean nearest-neighbor query. Hence one can construct a $p$-wspm k-NN graph in $O(nk^{2}\mathcal{T}_{Enn})$. As we typically have $k \ll n$, {\em i.e.} $k = O(\log(n))$ or even $k = O(1)$, this means that constructing a $p$-wspm k-NN graph requires only marginally more time than constructing a Euclidean $k$-NN graph (which requires $O(nk\mathcal{T}_{Enn})$).
    
    \item We verify experimentally that using a $p$-wspm in lieu of the Euclidean metric results in an appreciable increase in clustering accuracy, at the cost of a small increase in run time, for a wide range of real and synthetic data sets. 
    
\end{enumerate}

After establishing notation and surveying the literature in \S 2, we prove our main results in \S 3 and \S 4. In \S 5 we present our algorithm for computing $k$ nearest neighbors in any $p$-wspm, while in \S 6 we report the results of our numerical experiments.

\section{Definitions and Notation}
Let us first fix some notation. Throughout this paper, $D$ shall denote the ambient dimension, while $\mathcal{X}$ will denote a fixed, finite sets of distinct points in $\mathbb{R}^{D}$. We shall denote the Euclidean ({\em i.e.} $\ell_2$) norm on $\mathbb{R}^{D}$ as $\|\cdot\|$. For any finite set $S$, by $|S|$ we shall mean its cardinality. For any positive integer $\ell$, by $[\ell]$ we mean the set $\{1,2,\ldots, \ell\}$. Finally, for two functions $f(n)$ and $g(n)$ by $f(n) = O(g(n))$ we shall mean that there exist constants $C$ and $n_0$ such that $f(n) \leq Cg(n)$ for all $n\geq n_0$. Similarly, by $f(n) = o(g(n))$ we shall mean that $f(n)/g(n)\to 0$ as $n\to \infty$. Occasionally we shall explicitly indicate the dependence on the variable $n$ (later, also $n_a$) by writing $O_{n}$ or $o_n$. 

\subsection{Data Model}
\label{sec:DataModel}
We consider data sets $\mathcal{X} = \mathcal{X}_1\cup\mathcal{X}_{2}\cup\ldots \cup\mathcal{X}_{\ell} \subset \mathbb{R}^{D}$ consisting naturally of $\ell$ clusters, which are {\em a priori} unknown. Let $|\mathcal{X}_a| = n_a$. We posit that for each $\mathcal{X}_a$ there is a smooth, compact, embedded manifold $\mathcal{M}_{a} \hookrightarrow \mathbb{R}^{D}$ such that $\mathcal{X}_a\subset\mathcal{M}_a$. Let $g_{a}$ denote the restriction of the Euclidean metric to $\mathcal{M}_a$, then $(\mathcal{M}_a,g_{a})$ is a compact Riemannian manifold. We shall further assume that  $\mathcal{X}_a$ is sampled according to a probability density function $\mu_a$ supported on $\mathcal{M}_{a}$ and continuous with respect to the metric $g_{a}$. For any $\bfx,\bfy\in\mathcal{M}_a$ let 
\begin{equation*}
\text{dist}_{a}(\bfx,\bfy) := \inf_{\gamma} \int_{0}^{1} \sqrt{g_{a}(\gamma^{'}(t),\gamma^{'}(t))}dt
%\label{eq:DefnofDist_a}
\end{equation*}
denote the metric induced by $g_{a}$, where the infimum is over all piecewise smooth curves $\gamma : [0,1] \to \mathcal{M}_a$ with $\gamma(0) = \bfx$ and $\gamma(1) = \bfy$. Define the {\em diameter} of $\mathcal{M}_a$ to be the supremum over all distances between points in $\mathcal{M}_a$:
$$
\text{diam}(\mathcal{M}_a) := \sup_{\bfx,\bfy\in\mathcal{M}_a} \text{dist}_a(\bfx,\bfy)
$$
Since each $\mathcal{M}_a$ is compact this supremum is in fact a maximum and $\text{diam}(\mathcal{M}_a)$ is finite. We assume that the data manifolds are fairly well separated, that is, 
\begin{equation}
\text{dist}(\mathcal{M}_a,\mathcal{M}_b)  = \min_{\bfx\in\mathcal{M}_a, \bfy\in\mathcal{M}_b}\|\bfx - \bfy\| \geq \delta > 0 \text{ for all},\, 1\leq a< b\leq \ell
\label{eq:DefinitionOfDelta}
\end{equation}

Note that frequently (for example, in \cite{Arias2011}), this model is extended to allow for noisy sampling, whereby for some $\tau > 0$, $\mathcal{X}_a$ is sampled from the tube $B(\mathcal{M}_a,\tau)$ defined as:
$$
B(\mathcal{M}_a,\tau) = \left\{\bfx\in\mathbb{R}^{D} \ : \ \min_{\bfy\in\mathcal{M}_a}\|\bfx-\bfy\|_2 \leq \tau\right\}.
$$
but we leave this extension to future work. 
\subsection{Power Weighted Shortest Path Metrics}
\label{section:PowerWeightedSPM}

For any distinct pair $\bfx_{\alpha},\bfx_{\beta} \in \mathcal{X}$ and any path $\gamma = \bfx_{\alpha}\to \bfx_1\to\ldots\to \bfx_{m}\to\bfx_{\beta}$, define the $p$-weighted length of $\gamma$ to be:
\begin{equation}
L^{(p)}(\gamma) := \left(\sum_{j=0}^{m} \|\bfx_{i_{j+1}} - \bfx_{i_{j}}\|^{p}\right)^{1/p}
\label{eq:DefinitionOfPathLength}
\end{equation}
where by convention we declare $\bfx_{i_{0}} = \bfx_{\alpha}$ and $\bfx_{i_{m+1}} = \bfx_{\beta}$. We define the {\em $p$-weighted shortest path distance} from $\bfx_{\alpha}$ to $\bfx_{\beta}$ through $\mathcal{X}$ to be the minimum length over all such paths:
\begin{equation}
d^{(p)}_{\mathcal{X}}(\bfx_{\alpha},\bfx_{\beta}) := \min\left\{L^{(p)}(\gamma) \ : \ \gamma \text{ a path from $\bfx_{\alpha}$ to $\bfx_{\beta}$ through } \mathcal{X} \right\}
\label{eq:DefinitionOfdp}
\end{equation}
Note that $d^{(p)}_{\mathcal{X}}$ is a metric on the set $\mathcal{X}$ for $p\geq 1$ (see, for example, \cite{Howard2001}). As several authors \cite{Hwang2016,Chu2017,Alamgir2012} have noted, the metric $d^{(p)}_{\mathcal{X}}$ is {\em density-dependent}, so that if $\bfx_{\alpha}$ and $\bfx_{\beta}$ are contained in a region of high density ({\em i.e.} a cluster) the path distance $d^{(p)}(\bfx_{\alpha},\bfx_{\beta})$ will likely be shorter than the Euclidean distance $\|\bfx_{\alpha} - \bfx_{\beta}\|$ (as long as $p > 1$). 

\subsection{Longest-Leg Path Distance}
\label{sec:LLPD}
Another common path-based distance is the longest-leg path distance (LLPD), which we shall denote as $d^{(\infty)}_{\mathcal{X}}$ (the choice of this notation should become clear shortly). It is defined as the minimum, over all paths from $\bfx_{\alpha}$ to $\bfx_{\beta}$ through $\mathcal{X}$, of the maximum distance between consecutive points in the path ({\em i.e.} legs). Before formally defining $d^{(\infty)}_{\mathcal{X}}$, define, for any path $\gamma$ from $\bfx_{\alpha}$ to $\bfx_{\beta}$ through $\mathcal{X}$, {\em the longest-leg length of} $\gamma$ as:
$$
L^{(\infty)}(\gamma) = \max_{j=0,\ldots,m} \|\bfx_{i_{j+1}} - \bfx_{i_j}\|
$$
again we are using the convention that $\bfx_{i_{0}} = \bfx_{\alpha}$ and $\bfx_{i_{m+1}} = \bfx_{\beta}$. Now, in analogy with \eqref{eq:DefinitionOfdp}:
\begin{equation}
d^{(\infty)}_{\mathcal{X}}(\bfx_i,\bfx_j) = \min \left\{L^{(\infty)}(\gamma): \gamma \text{ a path from $\bfx_{\alpha}$ to $\bfx_{\beta}$ through $\mathcal{X}$} \right\} 
\end{equation}
$d^{\infty}_{\mathcal{X}}$ is also a metric, in fact an ultrametric \cite{Little2017}, on $\mathcal{X}$.

\subsection{Prior work}
\label{sec:PriorWork}
The idea of using $p$-wspm's for clustering was proposed in \cite{Vincent2003}, and further explored in \cite{Orlitsky2005}. Recently, several papers \cite{Fischer2003,Chang2008,Little2017} have considered the use of LLPD for clustering and in particular \cite{Little2017} provides performance guarantees for spectral clustering with LLPD for a data model that is similar to ours. \cite{Chu2017} studies $p$-wspm's for $p\geq 2$ and proposes to use such metrics with density-based clustering algorithms, such as DBScan, although they do not provide any experimental results. The paper \cite{Bijral2011} proposes the use of $p$-wspm's for semi-supervised learning and provides a fast Dijkstra-style algorithm for finding, for every $\bfx\in\mathcal{X}$, its nearest neighbor, with respect to a $p$-wspm, in some set of labeled data points $\mathcal{L}$. They consider a similar data model to ours, but do not provide any quantitative results on the behaviour of shortest path distances. More generally, shortest path distances are a core part of the ISOMAP dimension reduction algorithm \cite{Tenenbaum2000}, although we emphasize that here not all paths through $\mathcal{X}$ are considered---{\em first} a $k$-NN graph $G^{(k)}$ is computed from $\mathcal{X}$ and only paths in this graph are admissible. \\

The asympotic behaviour of power weighted shortest path distances are analyzed for Euclidean Poisson processes in \cite{Howard2001} and for points sampled from an arbitrary probability distribution supported on a Riemannian manifold $\mathcal{M}$ in \cite{Hwang2016}. Note that in \cite{Hwang2016} the lengths of the legs of the path are measured using geodesic distance on $\mathcal{M}$, which is not computable in the data model we are considering as the $\mathcal{M}_a$ are unknown. Finally, in \cite{Chu2017} the results of \cite{Hwang2016} are used to show that, for certain $k$ and $p \geq 2$, with high probability the Euclidean $k$-NN graph can be used to compute $p$-wspm distances in the case where the data is sampled from a single Riemannian manifold $\mathcal{M}$. We discuss this further in \S \ref{sec:Compare_with_Chu}. \\

%Finally, we mention that in \cite{Chu2017} the $p$-wspm with $p=2$ is studied, and some interesting connections between this distance and the nearest-neighbor geodesic distance are discovered. However, the applications of this distance to clustering is not explicitly explored. \\

On the computational side, we are unaware of any prior mention of Algorithm \ref{alg:Dijkstra2} in the literature, although similar algorithms, which solve slightly different problems, are presented in \cite{Har2016}, \cite{Moscovich2017} and \cite{Bijral2011}. In \cite{Bijral2011} a variation of Dijkstra's algorithm is presented which starts from a set of labeled data $\mathcal{L}$ and expands its search through the rest of the data $\mathcal{X}$. The algorithm is designed to find, for all $\bfx_i\in\mathcal{X}$, the nearest point in $\mathcal{L}$ to $\bfx_i$ with respect to a $p$-wspm and then terminate. Thus, it is not clear how one would extend this algorithm to finding $k$ nearest neighbors for $k >1$. The algorithm of \cite{Har2016} is formulated for any weighted graph $G = (V,E,A)$ ({\em i.e.} not just graphs obtained from data sets $\mathcal{X}\subset\mathbb{R}^{D}$) and as such is not well-adapted to the problem at hand. In particular, it has run time $O(k(n\log(n) + |E|))$. Because the distance graph obtained from $\mathcal{X}$ is implicitly complete, $|E| = O(n^{2})$ and this results in a run time proportional to $kn^{2}$, which is infeasible for large data sets. Finally, the algorithm presented in \cite{Moscovich2017}, although adapted to the situation of distance graphs of data sets, actually solves a slightly different problem. Specifically they consider finding the $k_1$ $p$-wspm nearest neighbors of each $\bfx\in\mathcal{X}$ {\em in a $k_2$ Euclidean nearest neighbors graph} of $\mathcal{X}$. As such, it is not clear whether the set of nearest neighbors produced by their algorithm are truly the $p$-wspm nearest neighbors in $\mathcal{X}$.\\

Let us also mention that our approach is ``one at a time'', whereas the other three algorithms mentioned are ``all at once''. That is, our algorithm takes as input $\bfx\in\mathcal{X}$ and outputs the $k$ $p$-wpsm nearest neighbors of $\bfx$. This can then be iterated to find the $p$-wspm nearest neighbors of all $\bfx\in\mathcal{X}$. In contrast, ``all at once'' algorithms directly return the sets of $k$ nearest neighbors for each $\bfx\in\mathcal{X}$. Thus it is possible our algorithm will have applications in other scenarios where the $p$-wspm nearest neighbors of only some small subset of points of $\mathcal{X}$ are required or in ``online'' scenarios where new data points are continuously received.

\section{Relation between $p$-wspm's for different values of $p$}
\label{sec:ElementaryResults}
Here we compare $p$-wspm's for different values of the power weighting, $p$, in the interval $[1,\infty]$.

\begin{theorem}
\label{thm:Compare_p_weighted_path_lengths}
For any fixed $\bfx_{\alpha},\bfx_{\beta}\in\mathcal{X}$, and any fixed path $\gamma$ from $\bfx_{\alpha}$ to $\bfx_{\beta}$:
\begin{enumerate}
    \item If $1\leq p < q < \infty$ then: 
    $$
    L^{(q)}(\gamma) \leq L^{(p)}(\gamma) \leq n^{(1/p-1/q)}L^{(q)}(\gamma)
    $$
    \item For any $1\leq p < \infty$:
    $$
      L^{(\infty)}(\gamma) \leq L^{(p)}(\gamma) \leq n^{1/p} L^{(\infty)}(\gamma) 
    $$
\end{enumerate}
where $n:=|\mathcal{X}|$.
\end{theorem}

\begin{proof}
We shall make use of the following well-known results from analysis: 
\begin{enumerate}
\item for any fixed $m$, any $v\in\mathbb{R}^{m}$ and any $0<p<q$ we have that 
\begin{equation}
\|v\|_{q} \leq \|v\|_p \leq m^{(1/p-1/q)}\|v\|_q
\label{eq:p_q_bound}
\end{equation}
\item for any fixed $m$, any $v\in\mathbb{R}^{m}$ and any $0<p<\infty$:
\begin{equation}
\|v\|_{\infty} \leq \|v\|_p \leq m^{(1/p)}\|v\|_{\infty}
\label{eq:p_inf_bound}
\end{equation}
\end{enumerate}

For any path $\gamma = \bfx_{\alpha}\to \bfx_1\to\ldots\to \bfx_{m}\to\bfx_{\beta}$, observe that $L^{(p)}(\gamma)$ can be thought of as $\|v_{\gamma}\|_p$, where $v_{\gamma}\in\mathbb{R}^{m+1}$ is the vector:
$$
v_{\gamma} = \left[\|\bfx_1 - \bfx_0\|,\|\bfx_2 - \bfx_1\|,\ldots, \|\bfx_{m+1} - \bfx_m\|\right]^{\top}
$$
It follows from \eqref{eq:p_q_bound} that $L^{(q)}(\gamma) \leq L^{(p)}(\gamma) \leq (m+1)^{(1/p-1/q)}L^{(q)}(\gamma)$. Observe that $m+1$, the path length, is always less than or equal to $n$, the number of points in $\mathcal{X}$. Thus $(m+1)^{(1/p-1/q)} \leq n^{(1/p-1/q)}$ and part 1. of the theorem follows. In a similar fashion, we obtain part 2. from \eqref{eq:p_inf_bound}. 
\end{proof}

\begin{corollary}
\label{cor:Monotonicity_of_path_dist}
For any fixed $\mathcal{X}$, any $1\leq p < q<\infty $ and all $\bfx_{\alpha},\bfx_{\beta}\in\mathcal{X}$:
$$
 d_{\mathcal{X}}^{(q)}(\bfx_{\alpha},\bfx_{\beta}) \leq d_{\mathcal{X}}^{(p)}(\bfx_{\alpha},\bfx_{\beta}) \leq n^{(1/p-1/q)}d_{\mathcal{X}}^{(q)}(\bfx_{\alpha},\bfx_{\beta})
$$
where $n:=|\mathcal{X}|$.
\end{corollary}

\begin{proof}
By Theorem \ref{thm:Compare_p_weighted_path_lengths}:  $L^{(q)}(\gamma) \leq L^{(p)}(\gamma) \leq n^{(1/p-1/q)}L^{(q)}(\gamma)$ for all paths $\gamma$ between $\bfx_{\alpha}$ and $\bfx_{\beta}$, hence the same relationship must hold for the minimum over all paths from $\bfx_{\alpha}$ to $\bfx_{\beta}$:
$$
\min_{\gamma}L^{(q)}(\gamma) \leq \min_{\gamma} L^{(p)}(\gamma) \leq n^{(1/p-1/q)}\min_{\gamma} L^{(q)}(\gamma)
$$
which, appealing to the definition of $d_{\mathcal{X}}^{(p)}(\bfx_{\alpha},\bfx_{\beta})$ (see \eqref{eq:DefinitionOfdp}) yields the corollary.
\end{proof}

\begin{corollary}
\label{cor:Limit_of_path_dist}
For any fixed $\mathcal{X}$, we have that $\displaystyle \lim_{p\to\infty} d^{(p)}(\bfx_{\alpha},\bfx_{\beta}) = d^{(\infty)}(\bfx_{\alpha},\bfx_{\beta})$ for all $\bfx_{\alpha},\bfx_{\beta}\in\mathcal{X}$.
\end{corollary}

\begin{proof}
Again by Theorem \ref{thm:Compare_p_weighted_path_lengths} we have that $L^{(\infty)}(\gamma) \leq L^{(p)}(\gamma) \leq n^{(1/p)}L^{(\infty)}(\gamma)$ for all paths $\gamma$ between $\bfx_{\alpha}$ and $\bfx_{\beta}$, hence by the same reasoning as in the proof of Corollary \ref{cor:Monotonicity_of_path_dist}:
$$
 d_{\mathcal{X}}^{(\infty)}(\bfx_{\alpha},\bfx_{\beta}) \leq d_{\mathcal{X}}^{(p)}(\bfx_{\alpha},\bfx_{\beta}) \leq n^{(1/p)}d_{\mathcal{X}}^{(\infty)}(\bfx_{\alpha},\bfx_{\beta}).
$$
Because this holds for all $p<\infty$, taking the limit we get:
$$
 d_{\mathcal{X}}^{(\infty)}(\bfx_{\alpha},\bfx_{\beta}) \leq \lim_{p\to\infty} d_{\mathcal{X}}^{(p)}(\bfx_{\alpha},\bfx_{\beta}) \leq \lim_{p\to\infty} n^{(1/p)}d_{\mathcal{X}}^{(\infty)}(\bfx_{\alpha},\bfx_{\beta})
$$
and the result follows from the fact that $\lim_{p\to\infty} n^{(1/p)} = 1$.
\end{proof}

% \begin{proof}
% We first claim that for any fixed path $\gamma$ from $\bfx_{\alpha}$ to $\bfx_{\beta}$ we have $\displaystyle \lim_{p\to\infty} L^{(p)}(\gamma) = L^{(\infty)}(\gamma)$. To see this, let us suppose that:
% $$
% \|\bfx_{i_{j^{*}}} - \bfx_{i_{j^{*}+1}}\| = \displaystyle \max_{j=0,\ldots,m} \|\bfx_{i_{j+1}} - \bfx_{i_j}\| =  L^{(\infty)}(\gamma).
% $$
% That is, $\bfx_{i_{j^{*}}}\to \bfx_{i_{j^{*}+1}}$ is the longest leg in $\gamma$. Then for any $p$:
% \begin{align*}
% L^{(p)}(\gamma) := \left(\sum_{j=0}^{m} \|\bfx_{i_{j+1}} - \bfx_{i_{j}}\|^{p}\right)^{1/p} & = \|\bfx_{i_{j^{*}}} - \bfx_{i_{j^{*}+1}}\|\left(1 + \sum_{j=0 \atop j\neq j^{*}}^{m} \left(\frac{\|\bfx_{i_{j+1}} - \bfx_{i_{j}}\|}{\|\bfx_{i_{j^{*}}} - \bfx_{i_{j^{*}+1}}\|}\right)^{p}\right)^{1/p} \\
% 	& \leq \|\bfx_{i_{j^{*}}} - \bfx_{i_{j^{*}+1}}\|\left(m^{1/p}\right) = L^{(\infty)}(\gamma)\left(m^{1/p}\right)
% \end{align*}
% Moreover, we always have that $L^{(\infty)}(\gamma) \leq L^{(p)}(\gamma)$, hence:
% $$
% L^{(\infty)}(\gamma) \leq L^{(p)}(\gamma) \leq L^{(\infty)}(\gamma)\left(m^{1/p}\right)
% $$
% thus the claim follows as $m^{1/p}\to 1$. Because the operation of taking a minimum is continuous, we get that:
% $$
% \lim_{p\to\infty}d^{(p)}(\bfx_{\alpha},\bfx_{\beta}) = \lim_{p\to\infty}\min_{\gamma}\left\{ L^{(p)}(\gamma) \right\} = \min_{\gamma}\left\{\lim_{p\to\infty}L^{(p)}(\gamma) \right\} = \min_{\gamma} \left\{L^{(\infty)}(\gamma) \right\} = d^{(\infty)}(\bfx_{\alpha},\bfx_{\beta})
% $$

% \end{proof}

\begin{theorem}
\label{lemma:d_(1)}
For all $\bfx_{\alpha},\bfx_{\beta}\in\mathcal{X}$, $d^{(1)}_{\mathcal{X}}(\bfx_{\alpha},\bfx_{\beta}) = \|\bfx_{\alpha} - \bfx_{\beta}\|$
\end{theorem}

\begin{proof}
$d^{(1)}$ is defined as a minimum over all paths from $\bfx_{\alpha}$ to $\bfx_{\beta}$ through $\mathcal{X}$, and in particular the one hop path $\gamma_{\alpha\to\beta} = \bfx_{\alpha}\to\bfx_{\beta}$ is such a path. We claim it is the shortest such path as for any other path $\gamma = \bfx_{\alpha} \to \bfx_{i_1}\to \ldots \to \bfx_{i_m} \to \bfx_{\beta}$ by repeated applications of the triangle inequality:
\begin{align*}
L^{(1)}(\gamma_{\alpha\to\beta}) = \|\bfx_{\alpha} - \bfx_{\beta}\| & = \| \bfx_{\alpha} - \sum_{j=1}^{m}(\bfx_{i_j} - \bfx_{i_{j}}) - \bfx_{\beta}\| \leq \sum_{j=0}^{m} \|\bfx_{i_j} - \bfx_{i_{j+1}}\| = L^{(1)}(\gamma)
\end{align*}
\end{proof}

\section{$p$-wspm's in the Multi-Manifold Setting}
\label{sec:Theory}
One of the most useful aspects of $p$-wspm's, when applied to clustering problems, is that they tend to ``squeeze'' points in the same cluster together, while (hopefully) keeping points in different clusters separated. Here we make this more precise. Specifically we show that for any $1 <p < \infty$ if the data comes from the model described in \S \ref{sec:DataModel} then:

\begin{itemize}

\item $\displaystyle \min_{\bfx_{\alpha}\in\mathcal{X}_a,\bfx_{\beta}\in\mathcal{X}_b} d^{(p)}_{\mathcal{X}}(\bfx_{\alpha},\bfx_{\beta}) \geq \delta > 0$ (see Lemma \ref{lemma:Bound_epsilon2}). Recall that $\delta$ is the minimal separation between data manifolds.

\item $\displaystyle \max_{a\in [k]} \max_{\bfx_{\alpha},\bfx_{\beta} \in \mathcal{X}_a} d^{(p)}_{\mathcal{X}}(\bfx_{\alpha},\bfx_{\beta}) =  O(n_{\min}^{(1-p)/pd_{\max}})$ with probability tending to $1$ as $n$ tends to $\infty$. (see Theorem \ref{theorem:ShortestPathsBound_eps1}).

\end{itemize}
 
In this section it is sometimes necessary to enlarge our definition of $p$-wspm to allow for paths between $\bfx,\bfy\in\mathbb{R}^{D}$ that are not necessarily in $\mathcal{X}$ (and points that are not in $\mathcal{X}$ shall be denoted without a subscript). Thus $d^{(p)}_{\mathcal{X}}(\bfx,\bfy)$ is technically defined as, using the notation of \S \ref{section:PowerWeightedSPM}, $d^{(p)}_{\mathcal{X}\cup\{\bfx,\bfy\}}(\bfx,\bfy)$.

\subsection{Paths between points in different clusters} 

Here we prove that $p$-wspm's maintain a separation between points in different clusters.

\begin{theorem}
Let $\epsilon_{2}$ denote the minimal distance between points in different clusters. That is:
$$
\epsilon_{2} := \min_{\substack{ a,b\in [\ell] \\ a\neq b}}\min_{\substack{ \bfx_{\alpha}\in\mathcal{X}_a \\ \bfx_{\beta}\in\mathcal{X}_b}} d^{(p)}_{\mathcal{X}}(\bfx_{\alpha},\bfx_{\beta})
$$
Then $\epsilon_{2} \geq \delta$ with $\delta$ as defined in \eqref{eq:DefinitionOfDelta}.
\label{lemma:Bound_epsilon2}
\end{theorem}

\begin{proof}
For any $\bfx_{\alpha}\in\mathcal{X}_a$ and $\bfx_{\beta}\in\mathcal{X}_{b}$ let $\gamma = \bfx_{\alpha}\to \bfx_{i_1}\to,\ldots, \to\bfx_{i_m}\to \bfx_{\beta}$ be any path from $\bfx_{\alpha}$ to $\bfx_{\beta}$ through $\mathcal{X}$, where again we are using the convention that $\bfx_{i_0}:= \bfx_{\alpha}$ and $\bfx_{i_{m+1}} = \bfx_{\beta}$. If $\bfx_{\alpha} \in\mathcal{X}_a$ and $\bfx_{\beta}\in\mathcal{X}_{b}$ there must exist (at least one) $j^{*}\in [m]$ such that $\bfx_{i_{j^{*}}}\in\mathcal{X}_a$ while $\bfx_{i_{j^{*}+1}}\in\mathcal{X}_b$. By the assumptions on the generative model, $\mathcal{X}_a\subset \mathcal{M}_a$ and $\mathcal{X}_b\subset \mathcal{M}_b$ and so: $\|\bfx_{i_{j^{*}+1}} - \bfx_{i_{j^{*}}}\|^{p} \geq \left(\text{dist}(\mathcal{M}_{a},\mathcal{M}_b)\right)^{p}  = \delta^{p}$ thus:
$$
L^{(p)}(\gamma) := \left(\sum_{j=0}^{m} \|\bfx_{i_{j+1}} - \bfx_{i_{j}}\|^{p}\right)^{1/p} \geq\left( \|\bfx_{i_{j^{*}+1}} - \bfx_{i_{j^{*}}}\|^{p}\right)^{1/p} \geq \delta.
$$
Because this holds for all such $\gamma$ we have $ d^{(p)}_{\mathcal{X}}(\bfx_{\alpha},\bfx_{\beta}) := \min_{\gamma}\left\{L^{(p)}(\gamma) \right\} \geq \delta$
and because this holds for all such $\bfx_{\alpha}$ and $\bfx_{\beta}$:
$$
\min_{\bfx_{\alpha}\in\mathcal{X}_a,\bfx_{\beta}\in\mathcal{X}_b} d^{(p)}_{\mathcal{X}}(\bfx_{\alpha},\bfx_{\beta}) \geq \delta
$$
Finally, this holds for all $a\neq b$, yielding the lemma.
\end{proof}

\subsection{Asymptotic Limits of power weighted shortest paths}
\label{sec:Asymptotics}
For all $a\in [\ell]$, define $d^{(p)}_{\mathcal{X}_a}(\bfx_{\alpha},\bfx_{\beta})$ as the minimum $p$-weighted length of paths from $\bfx_{\alpha}$ to $\bfx_{\beta}$ {\em through $\mathcal{X}_a$} ({\em i.e.} we are excluding paths that may pass through points in $\mathcal{X}\setminus \mathcal{X}_a$). Because $\mathcal{X}_a\subset\mathcal{X}$, it follows that $d^{(p)}_{\mathcal{X}}(\bfx_{\alpha},\bfx_{\beta}) \leq d^{(p)}_{\mathcal{X}_a}(\bfx_{\alpha},\bfx_{\beta})$\footnote{More generally the reader is invited to check that for any $\mathcal{Y}\subset\mathcal{X}$ we have that $d^{(p)}_{\mathcal{X}}(\bfx_{\alpha},\bfx_{\beta}) \leq d^{(p)}_{\mathcal{Y}}(\bfx_{\alpha},\bfx_{\beta})$.}. In this section we address the asymptotic behaviour of $d^{(p)}_{\mathcal{X}_a}(\bfx_{\alpha},\bfx_{\beta})$. Here is where we make critical use of the main theorem of \cite{Hwang2016}, which we state as Theorem \ref{theorem:FromHDH}. Recall that $\mu_{a}$ is the probability density function with respect to which $\mathcal{X}_a$ is sampled from $\mathcal{M}_a$, and that by assumption $\mu^{\min}_{a} := \displaystyle \min_{x\in\mathcal{M}_a}\mu_{a}(x) > 0$. Define the following power-weighted geodesic distance on $\mathcal{M}_a$:
\begin{equation}
\text{dist}^{(p)}_{a}(\bfx,\bfy) = \inf_{\eta} \int^{1}_{0}\frac{\sqrt{g_{a}(\eta_{t}^{'},\eta_{t}^{'})}}{\mu_{a}(\eta_t)^{(1-p)/d_{a}}}dt 
\label{eq:DefineWarpedDistance}
\end{equation}
where here the infimum is over all piecewise smooth paths $\eta: [0,1] \to \mathcal{M}_a$ with $\eta(0) = \bfx$ and $\eta(1) = \bfy$. As in \S \ref{sec:DataModel}, for the Riemannian manifold $(\mathcal{M}_{a},g_{a})$ let $\text{dist}_{a}(\bfx,\bfy)$ denotes the geodesic distance from $\bfx$ to $\bfy$ on $\mathcal{M}_{a}$ with respect to $g_a$. \\

In order to bound $d^{(p)}_{\mathcal{X}_a}(\bfx_{\alpha},\bfx_{\beta})$ we study an auxiliary shortest path distance, $d^{(p)}_{\mathcal{M}_a}$. This distance will also be defined as a minimum over $p$-weighted path lengths, but instead of measuring the length of the legs using the Euclidean distance $\|\cdot\|$, we measure them with respect to the geodesic distance $\text{dist}_{a}$:
\begin{equation}
d^{(p)}_{\mathcal{M}_a}(\bfx,\bfy) := \min_{\gamma}\left(\sum_{j=0}^{m}\text{dist}_{a}(\bfx_{i_{j+1}},\bfx_{i_j})^{p}\right)^{1/p} 
\label{eq:IntrinsicPathDist}
\end{equation}
where again the $\min$ is over all paths $\gamma$ from $\bfx$ to $\bfy$ through $\mathcal{X}_a$.

\begin{theorem}
\label{theorem:FromHDH}
Let $\mathcal{M}_{a}$ be a compact Riemannian manifold, and assume that $\mathcal{X}_{a}$ is drawn from $\mathcal{M}_a$ with continuous probability distribution $\mu_{a}$ satisfying $\min_{x\in\mathcal{M}_a}\mu_{a}(x) > 0$. Let $n_a:= |\mathcal{X}_a|$. For all $n_a$, let $r_{a} := n_a^{(1-p)/pd_a}$. Then for any $1\leq p < \infty$ and any fixed $\epsilon > 0$ there exists a constant $\theta^{'}_0$ independent of $n_a$ such that:
\begin{equation}
\mathbb{P}\left[\sup_{\substack{\bfx,\bfy \in \mathcal{M}_{a} \\ \text{dist}_{a}(\bfx,\bfy) \geq r_{a}}} \left|\frac{\left(d^{(p)}_{\mathcal{M}_a}(\bfx,\bfy)\right)^{p}}{n_{a}^{(1-p)/d_{a}}\text{dist}_{a}^{(p)}(\bfx,\bfy)} - C(d_{a},p)\right| > \epsilon \right]  = e^{-\theta_{0}^{'} n_{a}^{1/p(d_{a} + 2p)} + \mathcal{O}(\log(n_a))}
\label{eq:FromHDH}
\end{equation}
where $C(d_a,p)$ is a constant depending only on $d_{a}$ and $p$, but not on $n_a$.
\end{theorem}

\begin{remark}
This is a slightly modified version Theorem 1 in \cite{Hwang2016}. As stated in \cite{Hwang2016}, the $\sup$ is over $\bfx,\bfy\in\mathcal{M}$ satisfying $\text{dist}_{a}(\bfx,\bfy) \geq b$ for a fixed constant $b$. However immediately below the statement of Theorem 1 the authors acknowledge that one can weaken this assumption to $\text{dist}_{a}(\bfx,\bfy) \geq r_a$ as long as $n_ar_{a}^{d_a}/\log n_a \to \infty$, which is the case for our choice of $r_a$. Note that there is a slight notational discrepancy here. What is called $d^{(p)}_{\mathcal{X}}(\bfx_{\alpha},\bfx_{\beta})$ in \cite{Hwang2016} is our $\left(d^{(p)}_{\mathcal{M}_{a}}(\bfx_{\alpha},\bfx_{\beta})\right)^{p}$.
\end{remark}

\begin{corollary}
With assumptions as in Theorem \ref{theorem:FromHDH}, there exists a constant $\theta^{'}_{0}$, independent of $n_a$ such that:
$$
\mathbb{P}\left[\max_{\bfx_{\alpha},\bfx_{\beta}\in\mathcal{X}_{a}}d_{\mathcal{M}_a}^{(p)}(\bfx_{\alpha},\bfx_{\beta}) \leq C_{a}n_{a}^{(1-p)/pd_a} \right]  \geq 1 - \exp\left(-\theta_{0}^{'} n_{a}^{1/p(d_{a}+ 2p)} + \mathcal{O}(\log(n_a)\right)
$$
where $C_{a}$ is a constant depending on $d_{a},p,\mu_{a}^{\min}$ and $\text{diam}(\mathcal{M}_a)$ but not on $n_a$. 
\label{cor:IntrinsicBound}
\end{corollary}

\begin{proof}
First observe that because the because the one leg path $\gamma_{\alpha\to\beta} = \bfx_{\alpha}\to\bfx_{\beta}$ is trivially a path through $\mathcal{X}_a$, for any $\bfx_{\alpha},\bfx_{\beta}$ satisfying $\text{dist}_{a}(\bfx_{\alpha},\bfx_{\beta}) < r_a$ we have that:
$$
d^{(p)}_{\mathcal{M}_a}(\bfx_{\alpha},\bfx_{\beta}) \leq \left(\text{dist}_{a}(\bfx_{\alpha},\bfx_{\beta})^{p}\right)^{1/p} < r_a = n_{a}^{(1-p)/pd_a}
$$
Hence as long as $C_{a} \geq 1$ we get that:
$$
\mathbb{P}\left[\max_{\bfx_{\alpha},\bfx_{\beta}\in\mathcal{X}_{a}}d_{\mathcal{M}_a}^{(p)}(\bfx_{\alpha},\bfx_{\beta}) \leq C_{a}n_{a}^{(1-p)/pd_a} \right] = \mathbb{P}\left[\max_{\substack{\bfx_{\alpha},\bfx_{\beta}\in\mathcal{X}_{a} \\ \text{dist}_a(\bfx_{\alpha},\bfx_{\beta}) \geq r_a}}d_{\mathcal{M}_a}^{(p)}(\bfx_{\alpha},\bfx_{\beta}) \leq C_{a}n_{a}^{(1-p)/pd_a} \right]
$$
Because $\mathcal{X}_a\subset \mathcal{M}_a$ we may use Theorem \ref{theorem:FromHDH} to bound this probability. Indeed, fix any small $\epsilon < 1$. Then there exists a $\theta_0^{'}$ such that with probability at least $1 - \exp\left(-\theta_{0}^{'} n_{a}^{1/p(d_{a} + 2p)} + \mathcal{O}(\log(n_a)\right)$ we have:
\begin{align}
& \left(d^{(p)}_{\mathcal{M}_a}(\bfx_{\alpha},\bfx_{\beta})\right)^{p} \leq (C(d_{a},p) + \epsilon)n_{a}^{(1-p)/d_a}\text{dist}_{a}^{(p)}(\bfx_{\alpha},\bfx_{\beta}) \nonumber \\
\Rightarrow \ & d^{(p)}_{\mathcal{M}_a}(\bfx_{\alpha},\bfx_{\beta}) \leq \left[(C(d_{a},p) + \epsilon)\text{dist}_{a}^{(p)}(\bfx_{\alpha},\bfx_{\beta})\right]^{1/p}n_{a}^{(1-p)/pd_a} \label{eq:UpperBoundBracket} 
\end{align}
for all $\bfx_{\alpha},\bfx_{\beta}\in\mathcal{X}_a$ satisfying $\text{dist}_a(\bfx_{\alpha},\bfx_{\beta}) \geq r_a$. We now upper-bound the bracketed quantity in \eqref{eq:UpperBoundBracket}. From the definition of $\text{dist}_{a}^{(p)}$ (see \eqref{eq:DefineWarpedDistance})
\begin{align}
  \text{dist}_{a}^{(p)}(\bfx,\bfy) & \leq \frac{1}{(\mu_{a}^{\min})^{(p-1)/d_a}}\inf_{\eta}\int^{1}_{0}\sqrt{g_{a}(\eta_{t}^{'},\eta_{t}^{'})}dt  = \frac{1}{(\mu_{a}^{\min})^{(p-1)/d_a}}\text{dist}_{a}(\bfx,\bfy)
\label{eq:p_deformed_Distance}
\end{align}

Because $\mathcal{M}_a$ is compact and embedded, its diameter (see \S \ref{sec:DataModel}) is finite, and $\text{dist}_{a}(\bfx,\bfy) \leq \text{diam}(\mathcal{M}_a)$. So for all $\bfx_{\alpha},\bfx_{\beta}\in\mathcal{X}$ with $\text{dist}_a(\bfx_{\alpha},\bfx_{\beta}) \geq r_a$:
$$
\left[(C(d_{a},p) + \epsilon)\text{dist}_{a}^{(p)}(\bfx_{\alpha},\bfx_{\beta})\right]^{1/p} \leq (C(d_{a},p) + \epsilon)^{1/p}\frac{\text{diam}(\mathcal{M}_a)^{1/p}}{(\mu_{a}^{\min})^{(p-1)/pd_a}} =: \tilde{C}_a
$$
Defining $C_{a} = \max\{\tilde{C}_a,1\}$ we get that
$$
\mathbb{P}\left[\max_{\substack{\bfx_{\alpha},\bfx_{\beta}\in\mathcal{X}_{a} \\ \text{dist}_a(\bfx_{\alpha},\bfx_{\beta}) \geq r_a}}d^{(p)}_{\mathcal{M}_a}(\bfx_{\alpha},\bfx_{\beta}) \leq C_{a}n_{a}^{(1-p)/(pd_a)} \right] \geq 1 - \exp\left(-\theta_{0}^{'} n_{a}^{1/p(d_{a} + 2p)} + \mathcal{O}(\log(n_a)\right)
$$

thus proving the corollary. 
\end{proof}
Finally, it remains to compare the path distance with Euclidean legs, $d^{(p)}_{\mathcal{X}_a}$, to the path distance with geodesic legs, $d^{(p)}_{\mathcal{M}_a}$.

\begin{lemma}
For any $\bfx,\bfy\in \mathcal{M}_a$, and for all $a\in [k]$, $d^{(p)}_{\mathcal{X}_{a}}(\bfx,\bfy)\leq d^{(p)}_{\mathcal{M}_a}(\bfx,\bfy)$
\label{lemma:Compare_Euclidean_and_Riemannian}
\end{lemma}

\begin{proof}
Observe that for any $\bfx,\bfy\in\mathcal{M}_a$, $\|\bfx - \bfy\| \leq \text{dist}_{a}(\bfx,\bfy)$. It follows that for any path $\gamma = \bfx\to\bfx_{i_1}\to \ldots \to\bfx_{i_m}\to \bfy$ through $\mathcal{X}_a$:
$$
\sum_{j=0}^{m} \|\bfx_{i_{j+1}} - \bfx_{i_{j}}\|^{p} \leq \sum_{j=0}^{m}\text{dist}_{a}(\bfx_{i_{j+1}},\bfx_{i_j})^{p}
$$
and so:
\begin{align*}
\left(d^{(p)}_{\mathcal{X}_{a}}(\bfx,\bfy)\right)^{p} & = \min_{\gamma}\left\{\sum_{j=0}^{m} \|\bfx_{i_{j+1}} - \bfx_{i_{j}}\|^{p}\right\} \leq \min_{\gamma}\left\{\sum_{j=0}^{m}\text{dist}_{a}(\bfx_{i_{j+1}},\bfx_{i_j})^{p} \right\} = \left(d^{(p)}_{\mathcal{M}_a}(\bfx,\bfy)\right)^{p}
\end{align*}
whence the result follows.
\end{proof}

\subsection{Paths Between Points in the Same Cluster}
Let us now return to the full distance function $d^{(p)}_{\mathcal{X}}$.

\begin{theorem}
\label{theorem:ShortestPathsBound_eps1}
Define $\epsilon_1$ to be the maximal distance between points in the same cluster:
$$
\epsilon_{1} := \max_{a\in [\ell]} \max_{\bfx_{\alpha},\bfx_{\beta} \in \mathcal{X}_a} d^{(p)}_{\mathcal{X}}(\bfx_{\alpha},\bfx_{\beta})
$$
With assumptions as in \S \ref{sec:DataModel}, for any $1\leq p < \infty$:
$$
\mathbb{P}\left[ \epsilon_1 \leq C n^{(1-p)/pd_{\max}}\right] \geq 1 - \exp\left(-\theta_{0}^{'}n_{\min}^{1/p(d_{\max} + 2p)} + O(\log n)\right) 
$$
\end{theorem}

\begin{proof}
First, for all $\bfx_{\alpha},\bfx_{\beta}\in\mathcal{X}_a$ observe that:
\begin{equation}
\max_{\bfx_{\alpha},\bfx_{\beta}\in\mathcal{X}_a}d^{(p)}_{\mathcal{X}}(\bfx_{\alpha},\bfx_{\beta}) \leq \max_{\bfx_{\alpha},\bfx_{\beta}\in\mathcal{X}_a} d^{(p)}_{\mathcal{X}_{a}}(\bfx_{\alpha},\bfx_{\beta}) \leq \max_{\bfx_{\alpha},\bfx_{\beta}\in\mathcal{X}_a}d^{(p)}_{\mathcal{M}_a}(\bfx_{\alpha},\bfx_{\beta})
\label{eq:Schrute}
\end{equation}
where the first inequality is because $\mathcal{X}_a\subset\mathcal{X}$ and the second is Lemma \ref{lemma:Compare_Euclidean_and_Riemannian}. Now let $C:= \max_{a} C_a$. Clearly $C_{a}n_{a}^{(1-p)/pd_a} \leq Cn^{(1-p)/pd_{\max}}$ and similarly:
$$
\exp\left(-\theta_{0}^{'} n_{a}^{1/p(d_{a}+ 2p)} + \mathcal{O}(\log(n_a)\right) \leq \exp\left(-\theta_{0}^{'} n_{\min}^{1/p(d_{\max}+ 2p)} + \mathcal{O}(\log(n_a)\right)
$$

combining these observations, \eqref{eq:Schrute} and Corollary \ref{cor:IntrinsicBound}: 
$$
\mathbb{P}\left[\max_{\bfx_{\alpha},\bfx_{\beta}\in\mathcal{X}_{a}}d_{\mathcal{M}_a}^{(p)}(\bfx_{\alpha},\bfx_{\beta}) \leq Cn^{(1-p)/pd_{\max}} \right]  \geq 1 - \exp\left(-\theta_{0}^{'} n_{\min}^{1/p(d_{\max}+ 2p)} + \mathcal{O}(\log(n_a)\right)
$$
for all $a = 1,\ldots, k$. By the union bound, and the definition of $\epsilon_1$:
\begin{align*}
\mathbb{P}\left[\epsilon_1 \leq Cn^{(1-p)/pd_{\max}}\right] & \geq 1 - \exp\left(-\theta_{0}^{'} n_{\min}^{1/p(d_{\max}+ 2p)}\right)\sum_{a=1}^{k}\exp\left(\mathcal{O}(\log(n_a) \right) \\
& = 1 - \exp\left(-\theta_{0}^{'} n_{\min}^{1/p(d_{\max}+ 2p)} + \mathcal{O}(\log n)\right)
\end{align*}
\end{proof}

\subsection{Discussion}
Theorem \ref{theorem:ShortestPathsBound_eps1} reveals an interesting tradeoff, already present in the work of \cite{Howard2001,Hwang2016}, namely that by increasing $p$ we get a tighter upper bound on $\epsilon_1$, but it holds with lower probability. We find experimental evidence for this in \S \ref{sec:ExperimentalResults}. In \cite{Arias2011} and \cite{Little2017} bounds analogous to those in Theorems \ref{lemma:Bound_epsilon2} and \ref{theorem:ShortestPathsBound_eps1}, but for $\|\cdot\|$ and $d^{(\infty)}_{\mathcal{X}}$ respectively, are used to provide bounds on the performance of single-linkage heirarchical clustering and spectral clustering with a full similarity matrix. As the focus of this article is clustering with a $k$ nearest neighbors similarity matrix, we do not pursue this line of inquiry further here.

\section{A Fast Algorithm for $p$-wspm Nearest Neighbors}
\label{sec:FastDijkstra}
In this section we start from a more general perspective. Let $G = (V,E,A)$ be a weighted graph with weighted adjacency matrix $A$. Occasionally we shall find it more convenient not to fix an ordering of the vertices, in which case $A(u,v)$ will represent the weight of the edge $\{u,v\}$. We assume all edge weights are positive. For any $v\in V$ we denote its set of neighbors by $\mathcal{N}(v)$. By $\gamma = u \to w_{1} \to \ldots \to w_{m} \to v$ we shall mean the path from $u$ to $v$ in $G$ through $w_{1},\ldots, w_{m}$.
Here, this is only valid if $\{u,w_{1}\},\ldots, \{w_{i},w_{i+1}\},\ldots, \{w_{m},v\}$ are all edges in $G$. In analogy with \S \ref{section:PowerWeightedSPM} we maintain the convention that for such a path $\gamma$, $w_{0} = u$ and $w_{m+1} = v$. Define the {\em length} of $\gamma$ as the sum of all its edge weights: $L(\gamma) := \sum_{i=0}^{m} A(w_{i},w_{i+1})$ and similarly define the longest-leg length of $\gamma$ as: $L^{(\infty)}(\gamma) = \max_{i=0}^{m} A(w_{i},w_{i+1})$. For any $u,v \in V$ define the {\em shortest path distance} as:
$$
d_{G}(u,v) = \min \{ L(\gamma): \ \gamma \text{ a path from $u$ to $v$} \}
$$
and analogously define the longest-leg path distance as:
$$
d^{(\infty)}_{G}(u,v) = \min \{ L^{\infty}(\gamma): \ \gamma \text{ a path from $u$ to $v$} \}
$$

%We shall introduce a variant of Dijkstra's algorithm that finds the $k$ nearest neighbors of any $v\in V$ in the distances $d_{G}$ or $d^{\infty}_{G}$ in $O(k^{2})$ time. It is possible that this complexity could be brought down further with a more subtle implementation of the min-priority queue that is at the heart of Algorithm \ref{alg:Dijkstra2}. It follows that using Algorithm \ref{alg:Dijkstra2} one can find the $k$ nearest neighbors of all $v\in V$ in the distances $d_{G}$ or $d^{\infty}_{G}$ in $O(nk^{4})$ time. \\

\begin{definition}
Let $\mathcal{N}_{k,G}(v)$ denote the set of $k$ nearest neighbors of $v\in V$. That is, $\mathcal{N}_{k,G}(v) = \{w_1,\ldots, w_k\}$ with $A(v,w_1)\leq A(v,w_2) \leq \ldots \leq A(v,w_k) \leq A(v,w)$ for all $w \in V\setminus\{w_1,\ldots, w_k\}$. By convention, we take $v\in \mathcal{N}_{k,G}(v)$ 
\end{definition}

\begin{definition}
\label{def:kNN_Graph}
For any graph $G$, define a directed $k$ nearest neighbors graph $G^{(k)}$ with directed edges $(u,v)$ whenever $v \in \mathcal{N}_{k,G}(u)$.
\end{definition}
In practice we do not compute the entire edge set of $G^{(k)}$, but rather just compute the sets $\mathcal{N}_{k,G}(u)$ as it becomes necessary.

\begin{definition}
Let $\mathcal{N}^{d_{G}}_{k,G}(v)$ denote the set of $k$ vertices which are closest to $v$ {\em in the shortest-path distance} $d_{G}$. That is, $\mathcal{N}^{d_{G}}_{k,G}(v) = \{w_1,\ldots, w_k\}$ and $d_{G}(v,w_1)\leq d_{G}(v,w_2)\leq \ldots d_{G}(v,w_k)\leq d_{G}(v,w)$ for all $w\in V\setminus\mathcal{N}^{d_{G}}_{k,G}(v)$. By convention, we take $v$ to be in $\mathcal{N}^{d_{G}}_{k,G}(v)$. Similarly, we define $\mathcal{N}^{(\infty)}_{k,G}(v)$ to be the $k$ vertices closest to $v$ in the metric $d^{(\infty)}_{G}$.
\end{definition}

We have not specified how to break ties in the definition of $\mathcal{N}_{k,G}(v), \mathcal{N}^{d_{G}}_{k,G}(v)$ or $\mathcal{N}^{(\infty)}_{k,G}(v)$. For the results of this section to hold, any method will suffice, as long as we use the same method in all three cases. To simplify the exposition, we shall assume henceforth that all distances are distinct.

\begin{remark}
\label{remark:Relating_graph_dist_to_p_dist}
Let us relate this to the discussion in previous sections. For any set of data points $\mathcal{X} = \{\bfx_1,\ldots, \bfx_{n}\}\subset\mathbb{R}^{D}$ and any power weighting $1 \leq p < \infty$ one can form a graph $G$ on $n$ vertices, where the vertex $v_i$ corresponds to the data point $\bfx_{i}$, and edge weights $A_{ij} = \|\bfx_{i} - \bfx_{j}\|^{p}$. Then:
$$
d_{G}(v_{i},v_{j}) = \left(d^{(p)}_{\mathcal{X}}(\bfx_{i},\bfx_{j})\right)^{p}
$$
Moreover, if we denote by $\mathcal{N}^{(p)}_{k,\mathcal{X}}(\bfx_i)$ the $k$ nearest neighbors, with respect to $d^{(p)}_{\mathcal{X}}$, of $\bfx_i\in\mathcal{X}$, then:
$$
\mathcal{N}^{(p)}_{k,\mathcal{X}}(\bfx_i) = \mathcal{N}_{k,G}^{d_{G}}(v_i)
$$
The analogous results also hold for $d_{\mathcal{X}}^{(\infty)}$, {\em ie}
$$
d_{G}^{(\infty)}(v_{i},v_{j}) = d^{(\infty)}_{\mathcal{X}}(\bfx_{i},\bfx_{j}) \text{ and } \mathcal{N}^{(\infty)}_{k,\mathcal{X}}(\bfx_i) = \mathcal{N}_{k,G}^{(\infty)}(v_i)
$$
\end{remark}

%As mentioned earlier, Algorithm \ref{alg:Dijkstra2} finds $\mathcal{N}^{d_{G}}_{k,G}(v)$ (as well as $d_{G}(v,w)$ for all $w \in \mathcal{N}^{d_{G}}_{k,G}(v)$). Our modification may be known to experts (in \cite{Har2016} a similar variation of Dijkstra is referred to as `folklore') but to the best of the authors' knowledge has not appeared in the literature. Note that similar variations of Dijkstra's algorithm appear in \cite{Har2016} and \cite{Moscovich2017}, but both of these algorithms output the $k$-nearest neighbors {\em for all vertices}. Clearly, Algorithm \ref{alg:Dijkstra2} can be performed for each $v\in V$ to achieve this, but the algorithms presented in \cite{Har2016} and \cite{Moscovich2017} cannot be used to efficiently find $\mathcal{N}^{d_{G}}_{k,G}(v)$ for a fixed $v$ of interest.\\

Before proceeding, let us briefly review how Dijkstra's algorithm works. For a graph $G$ with non-negative edge weights, and a given source vertex, $s$, Dijkstra will return a list of pairs of the form $(u,d_{G}(s,u))$. The computational complexity of Dijkstra's algorithm is $\mathcal{O}(|E| + n\log n)$ so when $G$ is a complete graph the complexity is $\mathcal{O}(n^2)$. Intuitively, our proposed algorithm (Algorithm \ref{alg:Dijkstra2}) sidesteps this complexity by allowing one to run Dijkstra's algorithm on a much sparser graph, as long as one is only interested in determining the identity and path distance to the $k$ nearest neighbors of $s$, with respect to the path distance $d_{G}$ (see Lemma \ref{lemma:NN_Graph}).\\ 

The following implementation of Dijkstra's algorithm is as in \cite{Cormen2009}. The min-priority queue operations $\text{\tt decreaseKey}, \text{\tt insert}$ and $\text{\tt extractMin}$ have their standard definitions (see, for example Chpt. 6 of \cite{Cormen2009}). For any vertex $s\in V$ and any subset $W\subset V$, we shall also use the shorthand $\text{\tt makeQueue}(W,s)$ to denote the process of initializing a min-priority queue with $\text{\tt key}[s] = 0$ and $\text{\tt key}[v] = +\infty$ for all $v\in W\setminus s$. 
%Let  Their exact values depend on the implementation of $Q$, and will be discussed at the end of this section. Let us review Dijkstra's algorithm. The following implementation is as presented in \cite{Cormen2009}.\\

\begin{algorithm}
\caption{Dijkstra}
\label{alg:Dijkstra}
\begin{algorithmic}[1]
    \State {\bf Input:} weighted graph G, source vertex $s$.
    \State {\bf Output:} List $S$ containing $(u,d_{G}(s,u))$ for all $u\in V$.
	\State {\bf Initialize:} $Q \gets \text{\tt makeQueue}(V,s)$. Empty list $S$.
    \While{ $Q$ is non-empty}
    	\State $u \gets $ {\tt extractMin}(Q)
        \State Append $(u,\text{\tt key}[u])$ to $S$. \Comment{Once $u$ is extracted $\text{\tt key}[u]$ is shortest path length from $s$.}
        \For {$v\in \mathcal{N}(u)$} \Comment{ $\mathcal{N}(u)$ is the set of all vertices adjacent to $u$}
        \State $\text{tempDist} \gets \text{\tt key}[u] + A(u,v)$
        \If{$\text{tempDist} < \text{\tt key}[v]$}
        \State $\text{\tt key}[v] \gets \text{tempDist}$ \Comment{Update the distance from $s$ to $v$ if path through $u$ is shorter}
        \EndIf
        \EndFor
    \EndWhile
\State {\bf Output:} $S$

\end{algorithmic}
\end{algorithm}

Note once $u$ is popped in step $4$, $\text{\tt key}[u] = d_{G}(s,u)$. Our first key observation is the following:

\begin{lemma}
\label{lemma:CorrectPoppingOrder}
Suppose that all weights are non-negative: $A_{ij} \geq 0$. If $u_{i}$ is the $i$-th vertex to be removed from $Q$ at step $11$, then $u_{i}$ is the $i$-th closest vertex to $s$.
\end{lemma}

\begin{proof}
See, for example, the discussion in \cite{Cormen2009}.
\end{proof}

It follows that, if one is only interested in finding the $k$ nearest neighbors of $s$ in the path distance $d_{G}$, one need only iterate through the ``while'' loop $3\to 12$ $k$ times. There is a further inefficiency, which was also highlighted in \cite{Bijral2011}. The ``for'' loop 6--10 iterates over all neighbors of $u$. The graphs we are interested in are, implicitly, fully connected, hence this for loop iterates over all $n-1$ other vertices at each step. We fix this with the following observation:

\begin{lemma}
For any graph $G$, let $G^{(k)}$ denote its $k$-Nearest-Neighbor graph (see Definition \ref{def:kNN_Graph}). Then: 
$$
\mathcal{N}^{d_{G}}_{k,G}(v) = \mathcal{N}^{d_{G^{(k)}}}_{k,G^{(k)}}(v) \quad \text{ for all } v
$$
Note that in the directed graph $G^{(k)}$, we consider only paths that traverse each edge in the `correct' direction.
\label{lemma:NN_Graph}
\end{lemma}

Concretely: the path-nearest-neighbors in $G$ are the same as the path-nearest neighbors in $G^{(k)}$, hence one can find $\mathcal{N}^{d_{G}}_{k,G}(v)$ by running a Dijkstra-style algorithm on $G^{(k)}$, instead of $G$. As each vertex in $G^{(k)}$ has a small number of neighbors (precisely $k$), this alleviates the computational burden of the ``for'' loop 6--10 highlighted above. \\

Before proving this lemma, let us explain why it may seem counterintuitive. If $w \in \mathcal{N}^{d_{G}}_{k,G}(v)$ there is a path $\gamma$ from $v$ to $w$ that is short (at least shorter than the shortest paths to all $u \notin  \mathcal{N}^{d_{G}}_{k,G}(v)$). In forming $G^{(k)}$ from $G$, one deletes a lot of edges. Thus it is not clear that $\gamma$ is still a path in $G^{(k)}$ (some of its edges may now be ``missing''). Hence it would seem possible that $w$ is now far away from $v$ in the shortest-path distance in $G^{(k)}$. The lemma asserts that this cannot be the case. \\

\begin{proof}
Since the sets $\mathcal{N}^{d_{G}}_{k,G}(v)$ and  $\mathcal{N}^{d_{G^{(k)}}}_{k,G^{(k)}}(v)$ have the same cardinality ({\em i.e.} $k$), to prove equality it suffices to prove one containment. We shall show that $\mathcal{N}^{d_{G}}_{k,G}(v) \subset \mathcal{N}^{d_{G^{(k)}}}_{k,G^{(k)}}(v)$. Consider any $w\in \mathcal{N}^{d_{G}}_{k,G}(v)$. Let $\tilde{\gamma} = v\to u_1\to \ldots \to u_{m} \to w$ be a shortest path from $v$ to $w$. That is, $L(\tilde{\gamma}) = \min \{ L(\gamma): \gamma \text{ a path from $v$ to $w$} \}$. \\

{\em We claim that $\tilde{\gamma}$ is a path in $G^{(k)}$}. If this not the case, then there is an edge $\{u_{i},u_{i+1}\}$ that is in $\tilde{\gamma}$ but $(u_{i},u_{i+1})$ is not an edge in $G^{(k)}$ (we again adopt the convention that $u_{0} := v$ and $u_{m+1} := w$). By the construction of $G^{(k)}$ this implies that there are $k$ vertices $x_{1},\ldots, x_{k}$ that are closer to $u_{i}$ than $u_{i+1}$. (Note that the sets $\{u_{0},\ldots, u_{i-1}\}$ and $\{x_1,\ldots, x_k\}$ need not be disjoint). But then the paths $\gamma_{j} = v\to u_{1}\to \ldots \to u_{i} \to x_{j}$ in $G$ are all shorter than the path $v\to u_{1}\to \ldots \to u_{i+1}$ and hence shorter than $\tilde{\gamma}$, as all edge weights are assumed positive. It follows that $d_{G}(v,x_j) < d_{G}(v,w)$ for $j =1,\ldots, k$, contradicting the assumption that $w\in \mathcal{N}^{d_{G}}_{k,G}(v)$.\\

{\em Now, we claim that } $w \in \mathcal{N}^{d_{G^{(k)}}}_{k,G^{(k)}}(v)$. If this were not the case, there would exists $k$ other vertices $w_{1},\ldots, w_{k}$ that are closer in the shortest-path distance $d_{G^{(k)}}$ to $v$ than $w$. That is, there would be paths $\gamma_{1},\ldots, \gamma_{k}$ from $v$ to $w_{1},\ldots, w_{k}$ respectively that are shorter than $\gamma$. But every path in $G^{(k)}$ is also a path in $G$, hence $w_{1},\ldots, w_{k}$ are also closer to $v$ than $w$ in the shortest-path distance $d_{G}$. This contradicts the assumption that $w\in \mathcal{N}^{d_{G}}_{k,G}(v)$.\end{proof}

%The upshot of lemma \ref{lemma:NN_Graph} is that one may run Dijkstra's algorithm on $G^{(k)}$ instead of $G$ and as long as one is only interested in the $k$ nearest neighbors of $v$ in the shortest-path distance, still get the correct answer. \\

There is a final, minor, inefficiency in Algorithm \ref{alg:Dijkstra} that we can improve upon; $Q$ is initialized to contain all vertices $V$ when it is actually only necessary to initialize it to contain the neighbors of $s$. Combining these three insights we arrive at Algorithm \ref{alg:Dijkstra2}. We call this algorithm Dijkstra-with-pruning as the key idea, expressed in Lemma \ref{lemma:NN_Graph}, is to ``prune'' the neighborhood of each $v\in V$ down to the $k$ nearest neighbors of $v$. Note that we use $\text{\tt DecreaseOrInsert}$ as shorthand for the function that decreases $\text{\tt key}[v]$ to $\text{tempDist}$ if $\text{tempDist} < \text{\tt key}[v]$ and $v \in Q$, inserts $v$ into $Q$ with priority $\text{\tt key}[v] = \text{tempDist}$ if $v\notin Q$ and does nothing if $v \in Q$ but $\text{tempDist} \geq \text{\tt key}[v]$. In fact, this is equivalent to inserting a copy of $v$ into $Q$ with priority $\text{\tt key}[v] = \text{tempDist}$, hence $\text{\tt DecreaseOrInsert}$ has the same computational complexity as $\text{\tt insert}$ (see also \cite{Moscovich2017}). Note that in this implementation the size of $Q$ grows by one on every iteration of the inner for loop, 10--13. 

\begin{algorithm}
\caption{Dijkstra-with-pruning}
\label{alg:Dijkstra2}
\begin{algorithmic}[1]
    \State {\bf Input:} Graph $G$, source vertex $s$. 
    \State {\bf Output:} List $S$ containing $(v,d_{G}(s,v))$ for all $v \in \mathcal{N}^{d_{G}}_{G,k}(s)$.
    \State Compute $\mathcal{N}_{k,G}(s)$
	\State {\bf Initialize:} $Q\gets \text{\tt makeQueue}(\mathcal{N}_{k,G}(s),s)$. Empty list $S$.
    \For{i = 1:k}
    	\State $u \gets $ {\tt extractMin}(Q)
    	\State Append $(u, \text{\tt key}[u])$ to $S$
    	\State Compute $\mathcal{N}_{k,G}(u)$
        \For {$v\in \mathcal{N}_{k,G}(u)$} 
        	\State $\text{tempDist} \gets \text{\tt key}[u] + A(u,v)$
        	\State $\text{\tt DecreaseOrInsert}(v,\text{tempDist})$ 
        \EndFor
    \EndFor
\State {\bf Output:} $S$
\end{algorithmic}
\end{algorithm}

\begin{theorem}
For any $s$ and any $G$ with positive weights, Algorithm \ref{alg:Dijkstra2} is correct. That is, $S$ contains precisely the pairs $(v,d_{G}(s,v))$ for all $v \in \mathcal{N}^{d_{G}}_{G,k}(s)$. 
\label{thm:Dijktra2Correct}
\end{theorem}

\begin{proof}
By only using $\mathcal{N}_{k,G}(u)$ in step 8, Algorithm \ref{alg:Dijkstra2} is essentially running Dijkstra's algorithm on $G^{(k)}$. By Lemma \ref{lemma:CorrectPoppingOrder}, the first $k$ elements to be popped off the queue in line 9 are indeed the $k$ closest vertices to $s$ in the graph $G^{(k)}$. That is, $S$ contains $\left(v,d_{G}(s,v)\right)$ for all $v \in \mathcal{N}^{d_{G^{(k)}}}_{k, G^{(k)}}(s)$. By Lemma \ref{lemma:NN_Graph}, $\mathcal{N}^{d_{G^{(k)}}}_{k,G^{(k)}}(s) = \mathcal{N}^{d_{G}}_{k,G}(s)$. \end{proof}

\subsection{Analysis of complexity}
\label{sec:Complexity}
Let us determine the computational complexity of Algorithm \ref{alg:Dijkstra2}. We shall remain agnostic for the moment about the precise implementation of the min-priority queue, and hence shall use the symbols $\mathcal{T}_{in},\mathcal{T}_{dk}$ and $\mathcal{T}_{em}$ to denote the computational complexity of {\tt insert}, $\text{\tt decreaseKey}$ and {\tt extractMin} respectively.  As discussed above, the complexity of $\text{\tt DecreaseOrInsert}$ is also $\mathcal{T}_{in}$. Let $\mathcal{T}_{nn}$ denote the cost of a nearest neighbor query in $G$. Then the cost of a $k$ nearest neighbor query, {\em i.e.} the cost of determining $\mathcal{N}_{k,G}(u)$ as in line 8, is $k\mathcal{T}_{nn}$. \\

Initializing the queue in line 4 requires $k$ insertions, for a cost of $k\mathcal{T}_{in}$. Precisely $k$ {\tt extractMin} operations are performed, for a total cost of $k\mathcal{T}_{em}$. $\text{\tt DecreaseOrInsert}$ is performed $k^{2}$ times, for a cost of $k^{2}\mathcal{T}_{in}$. Finally $k+1$ $k$-Nearest Neighbor queries are performed, for a cost of $(k+1)k\mathcal{T}_{nn}$. This gives a total cost of $k\mathcal{T}_{em} + (k+k^{2})\mathcal{T}_{in} + (k^{2}+k)\mathcal{T}_{nn}$. If the min priority queue is implemented using a Fibonacci heap, {\tt insert} and {\tt decreaseKey} both run in constant time ({\em i.e.} $\mathcal{T}_{in}, \mathcal{T}_{dk} = O(1)$) while for {\tt extractMin} $\mathcal{T}_{em} = O(\log(|Q|))$. Note that $|Q|$ never exceeds $k^{2} + k$ as at most one element is added to $Q$ during every pass through the inner for loop 9--12, which happens $k^{2}$ times. Hence $\mathcal{T}_{em} = O(\log(k))$ and we have a net cost of $O(k\log(k) + k^{2}\mathcal{T}_{nn}))$ where $\mathcal{T}_{nn}$ depends on the specifics of $G$. \\

Let us return to the case of primary interest in this paper; where $G$ is the complete graph on $n$ vertices and $A_{ij} = \|\bfx_i - \bfx_j\|^{p}$. By remark \ref{remark:Relating_graph_dist_to_p_dist} to compute $\mathcal{N}^{(p)}_{k,\mathcal{X}}(\bfx_i)$ it will suffice to compute $\mathcal{N}_{k,G}^{d_{G}}(v_i)$. Here, $\mathcal{T}_{nn}$ is equal to the cost of a Euclidean nearest neighbors query on $\mathcal{X}$, namely $\mathcal{T}_{Enn}$. Because $\mathcal{T}_{Enn} \gg \log(k)/k$ we get that for this case Algorithm \ref{alg:Dijkstra2} runs in $O(k^{2}\mathcal{T}_{Enn})$, as advertised in the introduction. For a totally general data set, $\mathcal{T}_{Enn} = O(Dn)$. However if $\mathcal{X}$ is intrinsically low-dimensional, which we are assuming, it is possible to speed this up. For example if $\mathcal{X}$ is stored in an efficient data structure such as a k-d tree \cite{Bentley1975} or a cover tree \cite{Beygelzimer2006} then $\mathcal{T}_{Enn} = O(\log(n))$. Note that initializing a Cover tree requires $O(c^{d_{\max}}Dn\log(n))$ time, where $c$ is a fixed constant \cite{Beygelzimer2006}. Hence finding $\mathcal{N}^{(p)}_{k,\mathcal{X}}(\bfx_i)$ for all $\bfx_i\in \mathcal{X}$ requires $O(k^{2}n\log(n) + c^{d_{\max}}Dn\log(n))$.

\subsection{Extension to Longest-Leg Path Distance}
A small modification to Algorithm \ref{alg:Dijkstra2} allows one to compute the $k$ nearest neighbors in the longest-leg-path distance, simply change the `$+$' in line 10 to a `max'. This guarantees that $\text{tempDist}$ represents the longest-leg length of the path $s\to\ldots \to u \to v$. For completeness, we present this algorithm below as Algorithm \ref{alg:Dijkstra_for_LLPD}. The proof of correctness is analogous to Theorem \ref{thm:Dijktra2Correct}, and we leave it to the interested reader. 

\begin{algorithm}
\caption{Dijkstra-with-pruning for LLPD}
\label{alg:Dijkstra_for_LLPD}
\begin{algorithmic}[1]
    \State {\bf Input:} Graph $G$, source vertex $s$. 
    \State {\bf Output:} List $S$ containing $(v,d_{G}(s,v)$ for all $v \in \mathcal{N}^{d^{\infty}_{G}}_{G,k}(v)$.
    \State Compute $\mathcal{N}_{k,G}(s)$ 
	\State {\bf Initialize:} $Q\gets \text{\tt makeQueue}(\mathcal{N}_{k,G}(s),s)$. Empty list $S$.
    \For{i = 1:k}
    	\State $u \gets $ {\tt extractMin}(Q)
        \State Append $(u,\text{\tt key}[u])$ to $S$
        \State Compute $\mathcal{N}_{k,G}(u)$
        \For {$v\in \mathcal{N}_{G^{(k)}}(u)$}
        	\State $\text{tempDist} \gets \max\left\{\text{\tt key}[u],A(u,v)\right\}$
        	\State $\text{\tt DecreaseOrInsert}(v,\text{tempDist})$ 
        \EndFor
    \EndFor
\State {\bf Output:} $S$
\end{algorithmic}
\end{algorithm}

\begin{remark}
By the same arguments as in \S \ref{sec:Complexity}, for any $\mathcal{X} = \{\bfx_1,\ldots,\bfx_n\} \subset\mathbb{R}^{D}$ one may find $\mathcal{N}^{(\infty)}_{k,\mathcal{X}}(\bfx_i)$ by finding $\mathcal{N}^{(\infty)}_{k,G}(v_i)$ in the complete graph $G$ with edge weights $A_{ij} = \|\bfx_i - \bfx_j\|$. The complexity of Algorithm \ref{alg:Dijkstra_for_LLPD} is the same as Algorithm \ref{alg:Dijkstra2}, hence one may find $\mathcal{N}^{(\infty)}_{k,\mathcal{X}}(\bfx_i)$ for all $\bfx_i\in\mathcal{X}$ in $O(k^{2}n\log(n) + c^{d_{\max}}Dn\log(n))$.
\end{remark}

\subsection{Comparison with Results of \cite{Chu2017}}
\label{sec:Compare_with_Chu}
Recall the following definition:
\begin{definition}
Let $G = (V,E,A)$ be a weighted graph. A sub-graph $H\subset G$ is called a $1$-spanner of $G$ if $H$ has the same vertex set as $G$ and, for all $u,v\in V$ we have that $d_{G}(u,v) = d_{H}(u,v)$.
\end{definition}

When preparing this manuscript for publication, the authors became aware of the following result of Chu, Miller and Sheehy:
\begin{theorem}
Let $\mathcal{M}\subset \mathbb{R}^{D}$ be a compact Riemannian manifold of dimension $d$, and let $\mathcal{X} = \{\bfx_1,\ldots,\bfx_n\}$ be sampled from $\mathcal{M}$ according to a Lipschitz continuous probability distribution $\mu$ satisfying $\mu_{\min} > 0$. Let $G$ be the complete graph on $V = \mathcal{X}$ with $A_{ij} = \|\bfx_i - \bfx_j\|^{p}$. If $p\geq 2$ and $k = O(2^{d}\log(n))$ then with probability $1 - o(1)$ the Euclidean $k$ nearest neighbors graph $G^{(k)}$ is a $1$-spanner of $G$.
\label{thm:ChuSheehyMiller}
\end{theorem}

\begin{proof}
See Theorem 1.6 and Corollary 6.1.1 of \cite{Chu2017}.
\end{proof}

This theorem give another way to deduce Theorem \ref{thm:Dijktra2Correct}, as the statement that $G^{(k)}$ is a $1$-spanner of $G$ implies Lemma \ref{lemma:NN_Graph}. However Lemma \ref{lemma:NN_Graph} holds in more generality. In particular:
\begin{enumerate}
    \item It is not conditional ({\em ie} holds with probability $1$).
    \item It places no restriction on $k$.
    \item It holds for any data model. Note that Theorem \ref{thm:ChuSheehyMiller} does not hold for the Data Model of \S \ref{sec:DataModel}, where multiple manifolds are under consideration.
\end{enumerate}

\section{Numerical Experiments}
\label{sec:ExperimentalResults}

In this section we verify that using a $p$-wspm in lieu of the Euclidean distance does indeed result in more accurate clustering results, at a modest increase in run time. Specifically, we consider eight datasets, and compare the accuracy of spectral clustering using $k$-NN graphs constructed using $p$-wspm's for $p=2,10$ and $\infty$ and using the Euclidean metric. As a baseline, we also consider spectral clustering with a full similarity matrix based on the Euclidean metric. For notational reasons it is convenient to denote the Euclidean metric as $d^{(1)}_{\mathcal{X}}$ (which is correct by Theorem \ref{lemma:d_(1)}), in which case the four metrics under consideration are $d^{(1)}_{\mathcal{X}}, d^{(2)}_{\mathcal{X}}, d^{(10)}_{\mathcal{X}}$ and $d^{(\infty)}_{\mathcal{X}}$. All numerical experiments described in this section were implemented in {\tt MATLAB} on a mid 2012 Mac Pro with 2 2.4 GHz 6-Core Intel Xeon processors and 32 GB of RAM. All  code used is available at \url{danielmckenzie.github.io}.

\subsection{The Data Sets}

\paragraph{\bf Three Lines.}
We draw data uniformly from three horizontal line segments of length $5$ in the x-y plane, namely $y=0$, $y = 1$ and $y=2$. We draw $500$ points from each line to create three clusters. We then embed the data into $\mathbb{R}^{50}$ by appending zeros to the coordinates, and add i.i.d. random Gaussian noise to each coordinate (with standard deviation $\sigma = 0.14$).  

\paragraph{\bf Three Moons.} This data set is as described in \cite{Yin2018} and elsewhere. It has three clusters, generated by sampling points uniformly at random from the upper semi-circle of radius $1$ centered at $(0,0)$, the lower semi-circle of radius $1.5$ centered at $(1.5,0.4)$ and the upper semi-circle of radius $1$ centered at $(3,0)$. As for the Three Lines data set, we draw $500$ data points from each semi-circle, embed the data into $\mathbb{R}^{50}$ by appending zeros, and then add Gaussian noise to each coordinate with standard deviation $\sigma = 0.14$. 

\paragraph{\bf Three Circles.} Here we draw data points uniformly from three concentric circles, of radii $1,2.25$ and $3.5$. We draw $222$ points from the smallest circle, $500$ points from the middle circle and $778$ points from the largest circle (the numbers are chosen so that the total number of points is $1500$). As before, we embed this data into $\mathbb{R}^{50}$ and add i.i.d Gaussian noise to each component, this time with standard deviation of $\sigma = 0.14$. \\

Two dimensional projections of these data sets are shown in Figure \ref{fig:ShortestPaths_Synthetic_Data}. We also considered five real data sets. We focused on image data sets that are suspected to satisfy the manifold hypothesis, namely images of faces and objects taken from different angles and handwritten digits. We obtained most of our datasets from the UCI Machine Learning Repository \cite{Dua2019}. 

\paragraph{\bf DrivFace} consists of $80\times 80$ greyscale images of four drivers, from a variety of angles. There are $606$ images in total, and the largest class contains $179$ images while the smallest class contains $90$ images.\footnote{available at \url{https://archive.ics.uci.edu/ml/datasets/DrivFace} and see also \cite{Diaz2016}}

\paragraph{\bf COIL-20} The Columbia Object Image Library (COIL) contains greyscale images of a variety of objects. There are $72$ images of each object, all from different angles. The COIL-20 dataset contains all $72$ images for $20$ objects, for a total of $1440$ images.\footnote{available at \url{http://www.cs.columbia.edu/CAVE/software/softlib/coil-20.php} and see also \cite{Nene1996}}.

\paragraph{\bf OptDigits} This data set consists of downsampled, $8\times 8$ greyscale images of handwritten digits $0-9$ and is available at . There are $150$ images of zero, and approximately $550$ images each of the remaining digits, for a total of $5620$ images.\footnote{available at \url{ https://archive.ics.uci.edu/ml/datasets/optical+recognition+of+handwritten+digits}} 

\paragraph{\bf USPS}
This data set consists of $16\times 16$, greyscale images of the handwritten digits 0--9. There are $1100$ images per class for a total of $11\ 000$ images.\footnote{available at: \url{https://cs.nyu.edu/~roweis/data.html}} 

\paragraph{\bf MNIST} This data set consists of $28\times 28$ greyscale images of the handwritten digits 0--9. We combined the test and training sets to get a total of $70 \ 000$ images.\footnote{available at \url{http://yann.lecun.com/exdb/mnist/}}

\begin{figure}
\centering
\minipage{.32\textwidth}
  \includegraphics[width =\linewidth]{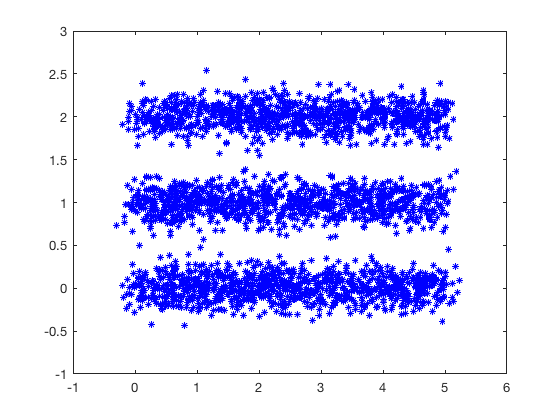}
\endminipage\hfill
\minipage{.32\textwidth}
\includegraphics[width =\linewidth]{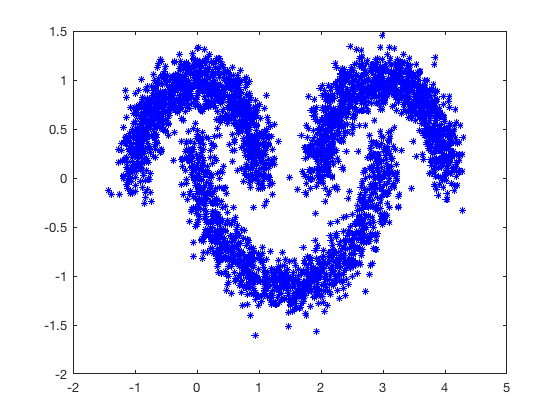}
\endminipage\hfill
\minipage{.32\textwidth}%
  \includegraphics[width = \linewidth]{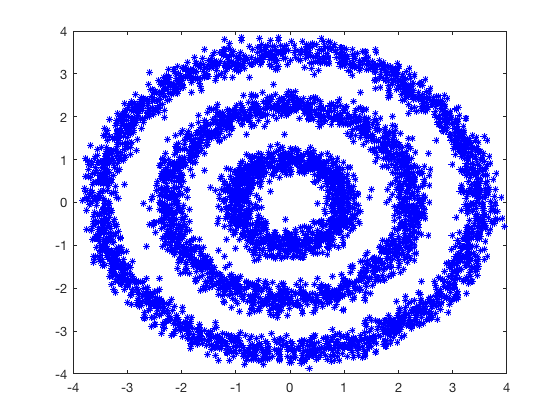}
\endminipage
\caption{All three synthetic data sets, projected into $\mathbb{R}^{2}$. From left to right: Three Lines, Three Moons and Three Circles.}
\label{fig:ShortestPaths_Synthetic_Data}
\end{figure}

\subsection{Preprocessing the Data}
Zelnik-Manor and Perona \cite{Zelnik2005} propose the following, locally scaled kernel for spectral clustering:
$$
A_{ij} = \exp\left( - d^{2}(\bfx_i,\bfx_j)/\sigma_i\sigma_j\right)
$$
where $\sigma_i:= d(\bfx_i,\bfx_{[r,i]})$ and $\bfx_{[r,i]}$ denotes the $r$-th closest point in $\mathcal{X}$ to $\bfx_i$. It is shown in \cite{Zelnik2005} and elsewhere that this kernel tends to outperform the unscaled Gaussian kernel (and indeed most other choices of kernel function), hence we adopt this for our experiments. We construct the full Euclidean similarity matrix as:
$$
A^{(f,1)}_{ij} = \exp\left( - \|\bfx_i - \bfx_j\|^{2}/\sigma_i\sigma_j\right)
$$
with $r = 10$. We also construct weighted $k$-NN similarity matrices for all four metrics using the same kernel and described in detail as Algorithm \ref{alg:Construct_Sim_Mat}. All the real data sets were initially vectorized, so the $80\times 80$ {\tt DrivFace} data set becomes vectors in $\mathbb{R}^{6400}$ and so on. Note that this procedure for constructing weighted $k$-NN similarity matrices was inspired by \cite{Jacobs2018}.

\begin{algorithm}
\caption{Construct Similarity Matrix}
\label{alg:Construct_Sim_Mat}
\begin{algorithmic}[1]
    \State {\bf Input:} Parameters $r,k,p$. Data set $\mathcal{X}\subset\mathbb{R}^{D}$. 
    \State {\bf Output:} weighted $k$-NN similarity matrix $A^{(p)}$.
    \State $n\gets |\mathcal{X}|$
    \State For $i=1,\ldots, n$ compute $\mathcal{N}_{k,\mathcal{X}}^{(p)}(\bfx_i)$ using Algorithm \ref{alg:Dijkstra2} or \ref{alg:Dijkstra_for_LLPD}.
    \State Compute $\sigma^{(p)}_i:= d^{(p)}_{\mathcal{X}}(\bfx_i,\bfx_{[r,i]})$, where $\bfx_{[r,i]}$ denotes the $r$-th closest point in $\mathcal{X}$ to $\bfx_i$ with respect to the distance $d^{(p)}_{\mathcal{X}}$.
    \State Define $\tilde{A}^{(p)}$ as: $\tilde{A}^{(p)}_{ij} = \left\{\begin{array}{cc}\exp\left(-d^{(p)}_{\mathcal{X}}(\bfx_i,\bfx_j)^{2}/\sigma_i\sigma_j\right) & \text{ if } \bfx_j \in \mathcal{N}^{(p)}_{k,\mathcal{X}}(\bfx_i)\\ 0 & \text{otherwise} \end{array}\right.$
    \State Symmetrize: $A^{(p)}_{ij} \gets \max\left\{\tilde{A}^{(p)}_{ij},\tilde{A}^{(p)}_{ji}\right\}$
\State {\bf Output:} $A^{(p)}_{ij}$
\end{algorithmic}
\end{algorithm}

% \begin{itemize}
% \item Fix parameters $r=10$ and $k=15$.

% \item For $\xi = 1,2,10,\infty$, and for all $i\in [n]$, let $\sigma^{(\xi)}_i:= d^{(\xi)}_{\mathcal{X}}(\bfx_i,\bfx_{[r,i]})$, where $\bfx_{[r,i]}$ denotes the $r$-th closest point in $\mathcal{X}$ to $\bfx_i$ with respect to the distance $d^{(\xi)}_{\mathcal{X}}$. Let $\text{NN}^{(\xi)}(\bfx_i,k)\subset \mathcal{X}$ denote the set of the $k$ closest points in $\mathcal{X}$ to $\bfx_i$ with respect to $d^{(\xi)}_{\mathcal{X}}$.

% \item 

% \item Symmetrize by defining $A^{(\xi)}_{ij} = \max\left\{\tilde{A}^{(\xi)}_{ij},\tilde{A}^{(\xi)}_{ji}\right\}$ 
% \end{itemize}

\begin{remark}
We certainly make no claim that the choice of parameters $r=10$ and $k=15$ is optimal, and indeed playing around with them can result in slightly better (or worse) results on certain data sets. However, we observed that changing the parameters had little qualitative effect on the results, and in particular on the ordering of the similarity matrices from least to most accurate (see Table \ref{table:Synthetic_Results}). As a sanity check, we also experimented with an unweighted k-NN graph, whereby for $p = 1,2,10$ and $\infty$ we define $A^{(\text{uw},p)}$ as:
$$
A^{(\text{uw},p)}_{ij} = \left\{\begin{array}{cc} 1 & \text{ if } \bfx_j \in \mathcal{N}^{(p)}_{k,\mathcal{X}}(\bfx_i) \text{ or } \bfx_i \in\mathcal{N}^{(p)}_{k,\mathcal{X}}(\bfx_j) \\ 0 & \text{otherwise} \end{array}\right.
$$
Again, the relative ordering of the results changed little, although the accuracy was several points lower for all four metrics. All code to reproduce the experiments is available on the second author's website, and we invite the curious reader to experiment further for themselves.
\end{remark}

\subsection{Experimental Results}
For each similarity matrix we perform normalized spectral clustering as described in Ng, Jordan and Weiss \cite{Ng2002} using freely available code  \footnote{See: \url{https://www.mathworks.com/matlabcentral/fileexchange/34412-fast-and-efficient-spectral-clustering}}. We compute $k$ nearest neighbors in the $p$-wspm's using Algorithms \ref{alg:Dijkstra2} and \ref{alg:Dijkstra_for_LLPD}, with the data points stored in a k-d tree. We calculate the accuracy by comparing the output of spectral clustering to the ground truth and recorded the running time. For the randomly generated data sets ({\em ie} Three Lines, Three Circles and Three Moons) we ran fifty independent trials and report the mean and standard deviation. For the deterministic data sets ({\em ie} all the others) we ran ten independent trials and report the mean. The results are  displayed in Tables 1 and 2. We do not attempt clustering with a full similarity matrix, $A^{(f,1)}$ on {\tt MNIST} as the resulting matrix is too large to hold in memory.\\

From these results, we may draw several broad conclusions. Observe that for smaller or low dimensional data sets, constructing the full similarity matrix $A^{(f,1)}$ is fastest. This is due to the overhead cost of constructing the k-d tree incurred by the nearest neighbors methods. This situation is reversed in higher dimensions or for larger data sets. The gap between the run-times for $A^{(1)}$ and $A^{(2)},A^{(10)}$ and for $A^{(\infty)}$ is attributable to the cost of running Algorithm \ref{alg:Dijkstra2} or \ref{alg:Dijkstra_for_LLPD}. Although this gap is large, relatively, for smaller data sets it becomes less relevant for larger data sets. In fact, the entire process of spectral clustering is {\em faster} with a p-wspm for the {\tt MNIST} data set. We attribute this to the fact that $A^{(2)},A^{(10)}$ and $A^{(\infty)}$ are more ``block-diagonal'' than $A^{(1)}$, hence finding their leading eigenvectors takes less time. This more than offsets the extra time required to construct them. With regards to accuracy, observe that for every data set clustering using the Euclidean metric (either the full version or the $k$-NN version) is less accurate then using a $p$-wspm. The catch here is that {\em which} $p$-wspm varies. As a general rule, $d^{(\infty)}_{\mathcal{X}}$ appears best for elongated data ({\em ie} the three lines) or when there is a large gap between the intrinsic dimension of the data and the ambient dimension ({\em ie} {\tt USPS} or {\tt MNIST}). On the other hand, $d^{(2)}_{\mathcal{X}}$ appears best for more globular data ({\em ie} the three moons) or when the gap between intrinsic and ambient dimension is less pronounced ({\em ie} {\tt OptDigits}). In all cases, $d^{(10)}_{\mathcal{X}}$ seems to be a good compromise between these two extremes.\\

Finally, note that the standard deviation of the accuracy is much higher for $d^{(10)}_{\mathcal{X}}$ (and $d^{(\infty)}_{\mathcal{X}}$) than it is for $d^{(1)}_{\mathcal{X}}$ or $d^{(2)}_{\mathcal{X}}$. This is in agreement with Theorem \ref{theorem:ShortestPathsBound_eps1}, where it is shown that the bound on $\epsilon_1$ holds with probability inversely proportional to $p$.

\begin{table}
    \centering
    \begin{tabular}{|c|ccccc|}
        \hline
                  & $A^{(f,1)}$ & $A^{(1)}$ & $A^{(2)}$ & $A^{(10)}$ & $A^{(\infty)}$ \\
        \hline
          3 Lines & $66.11\pm 0.94\%$ & $66.35 \pm 3.73\%$ & $66.87 \pm 3.37\%$ & $95.38\pm 9.22\%$ & $\mathbf{95.38 \pm 9.1\%}$ \\
         \hline
         3 Moons  & $85.90 \pm 1.13\%$ & $94.40 \pm 1.48\%$ & $94.40 \pm 1.48\%$ & $\mathbf{96.20 \pm 1.76\%}$ & $94.35 \pm 3.34\%$ \\
         \hline
         3 Circles & $51.87 \pm 0.00\%$ &$51.93 \pm 0.32\%$ & $51.94 \pm 0.36\%$ & $71.22 \pm 9.50\%$ & $\mathbf{73.61 \pm 10.47\%}$ \\
         \hline
         \hline
         {\tt DrivFace} & $78.88\%$ & $71.62\%$ & $71.62\%$ & $74.71\%$ & $\mathbf{85.38\%}$ \\
         \hline
         {\tt COIL-20} & $63.24\%$ & $75.28\%$ & $\mathbf{78.61\%}$ & $77.45\%$ & $60.92\%$ \\
         \hline
        {\tt OptDigits} & $77.73\%$ & $91.49\%$ & $\mathbf{91.54\%}$ & $88.39\%$ & $83.17\%$ \\
         \hline
          {\tt USPS} & $48.65\%$ & $65.05\%$ & $65.02\%$ & $76.20\%$ & $\mathbf{77.92\%}$ \\
          \hline
          {\tt MNIST} & - & $76.11\%$ & $75.63\%$ & $84.54\%$ & $\mathbf{86.77\%}$ \\
          \hline
    \end{tabular}
    \caption{Classification accuracy of spectral clustering. Note that $A^{(1)}$ represents using the Euclidean metric.}
    \label{table:Synthetic_Results}
\end{table}

\begin{table}
    \centering
    \begin{tabular}{|c|ccccc|}
        \hline
                  & $A^{(f,1)}$ & $A^{(1)}$ & $A^{(2)}$ & $A^{(10)}$ & $A^{(\infty)}$ \\
        \hline
        3 Lines   & $0.32$ & $0.16$ & $1.20$ & $1.22$ & $1.22$ \\
         \hline 
        3 Moons   & $0.33$ & $0.17$ & $1.31$ & $1.30$ & $1.36$ \\
          \hline
        3 Circles & $0.35$ & $0.16$ & $1.00$ & $1.06$ & $1.07$ \\
          \hline
          \hline
          {\tt DrivFace} & $0.37$ & $1.24$ & $1.55$& $1.64$ & $1.64$ \\
          \hline
          {\tt COIL-20} & $0.57$ & $0.72$ & $1.57$ & $1.82$ & $1.78$ \\
          \hline
        {\tt OptDigits} & $5.40$ & $1.41$ & $5.28$ & $5.58$ & $5.67$ \\
         \hline 
          {\tt USPS} & $27.40$ & $17.12$ & $26.75$ & $22.78$ & $23.79$ \\
          \hline
          {\tt MNIST} & - & $2060.23$ & $2031.38$ & $1554.15$ & $1613.41$ \\
          \hline 
    \end{tabular}
    \caption{Run time of spectral clustering, in seconds. Note that this includes the time rquired to construct the similarity matrix. $A^{(1)}$ represents using the Euclidean metric.}
    \label{table:Synthetic_Times}
\end{table}

\subsection{Varying the Power Weighting}
From the analysis of \S \ref{sec:Theory} it would appear that taking $p$ to be as large as possible always results in the best clustering results. However, this is true only in an asymptotic sense, and indeed the results contained in Table 1 indicate that for finite sample sizes this is not always the case. In Figure \ref{fig:Varying_p} we show the results of varying $p$ from $1$ to $20$ for the three lines data set, this time with $300$ points drawn from each cluster. We do this for three values of the ambient dimension, $D=10,50$ and $100$. As is clear, the optimal value of $p$ depends on the dimension\footnote{The observant reader will notice that we are only varying the ambient dimension, which according to the analysis of \S \ref{sec:Theory} should have no effect. Recall however, that we are adding Gaussian noise of the ambient dimension, which `thickens' the data manifolds and makes their intrinsic dimension weakly dependent on $D$.}. In particular, observe that an intermediate value of $p$, say $p=14$, is optimal when the ambient dimension is $50$ but that a smaller power weighting ($p=2$) is more appropriate when the ambient dimension is $10$. When the ambient dimension is $100$, no power weighting performs well, which is likely because for such a large value of $D$, and such a small amount of data per cluster, the noise drowns out any cluster structure.

\begin{figure}
 \centering
   \includegraphics[width =0.5\linewidth]{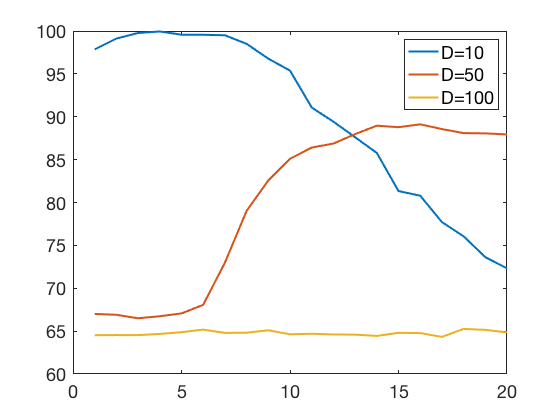}
   \caption{Varying $p$ and recording the accuracy of spectral clustering on the Three Lines data set, for three different values of the ambient dimension.}
   \label{fig:Varying_p}
\end{figure}

\section{Conclusions and Future Directions}
In this paper we argued that $p$-wspm's are well-suited to the problem of clustering high dimensional data when the data is sampled from a disjoint union of low dimensional manifolds. We showed that spectral clustering with a $p$-wspm outperforms spectral clustering with Euclidean distance, and using Algorithm \ref{alg:Dijkstra2} the increase in computational burden is negligible. From the results of \S\ref{sec:ExperimentalResults} it is clear that the geometry of the data manifolds influences which power weighting is optimal, and it would be of interest to analyze this further. Also of interest would be to extend our work to more general data models, for example those that only require the sampling distributions $\mu_a$ to be supported ``near'' $\mathcal{M}_a$, or those that allow for intersections between the data manifolds.

\section{Acknowledgements}
The first author gratefully acknowledges the support of the Department of Mathematics, The University of Georgia, where the first author was a graduate student while this work was completed. The second author thanks the Department of Mathematics, The University of Michigan for their support. Both authors thank the anonymous reviewer for many useful suggestions.

\providecommand{\href}[2]{#2}
\providecommand{\arxiv}[1]{\href{http://arxiv.org/abs/#1}{arXiv:#1}}
\providecommand{\url}[1]{\texttt{#1}}
\providecommand{\urlprefix}{URL }


\begin{thebibliography}{10}

\bibitem{Alamgir2012}
\newblock M.~Alamgir and U.~Von~Luxburg,
\newblock Shortest path distance in random k-nearest neighbor graphs,
\newblock \emph{arXiv preprint arXiv:1206.6381}.

\bibitem{Aldroubi2019}
\newblock A.~Aldroubi, K.~Hamm, A.~Koku and A.~Sekmen,
\newblock Cur decompositions, similarity matrices, and subspace clustering,
\newblock \emph{Front. Appl. Math. Stat. 4: 65. doi: 10.3389/fams}.

\bibitem{Arias2011}
\newblock E.~Arias-Castro,
\newblock Clustering based on pairwise distances when the data is of mixed
  dimensions,
\newblock \emph{IEEE Transactions on Information Theory}, \textbf{57} (2011),
  1692--1706.

\bibitem{Basri2003}
\newblock R.~Basri and D.~Jacobs,
\newblock Lambertian reflectance and linear subspaces,
\newblock \emph{IEEE Transactions on Pattern Analysis \& Machine Intelligence},
  218--233.

\bibitem{Bentley1975}
\newblock J.~L. Bentley,
\newblock Multidimensional binary search trees used for associative searching,
\newblock \emph{Communications of the ACM}, \textbf{18} (1975), 509--517.

\bibitem{Beygelzimer2006}
\newblock A.~Beygelzimer, S.~Kakade and J.~Langford,
\newblock Cover trees for nearest neighbor,
\newblock in \emph{Proceedings of the 23rd international conference on Machine
  learning},
\newblock ACM, 2006,
\newblock 97--104.

\bibitem{Bijral2011}
\newblock A.~Bijral, N.~Ratliff and N.~Srebro,
\newblock Semi-supervised learning with density based distances,
\newblock in \emph{Proceedings of the Twenty-Seventh Conference on Uncertainty
  in Artificial Intelligence},
\newblock AUAI Press, 2011,
\newblock 43--50.

\bibitem{Chang2008}
\newblock H.~Chang and D.-Y. Yeung,
\newblock Robust path-based spectral clustering,
\newblock \emph{Pattern Recognition}, \textbf{41} (2008), 191--203.

\bibitem{Chu2017}
\newblock T.~Chu, G.~Miller and D.~Sheehy,
\newblock Exploration of a graph-based density sensitive metric,
\newblock \emph{arXiv preprint arXiv:1709.07797}.

\bibitem{Coifman2006}
\newblock R.~Coifman and S.~Lafon,
\newblock Diffusion maps,
\newblock \emph{Applied and computational harmonic analysis}, \textbf{21}
  (2006), 5--30.

\bibitem{Cormen2009}
\newblock T.~Cormen, C.~Leiserson, R.~Rivest and C.~Stein,
\newblock \emph{Introduction to algorithms},
\newblock MIT press, 2009.

\bibitem{Costeira1998}
\newblock J.~Costeira and T.~Kanade,
\newblock A multibody factorization method for independently moving objects,
\newblock \emph{International Journal of Computer Vision}, \textbf{29} (1998),
  159--179.

\bibitem{Diaz2016}
\newblock K.~Diaz-Chito, A.~Hern{\'a}ndez-Sabat{\'e} and A.~L{\'o}pez,
\newblock A reduced feature set for driver head pose estimation,
\newblock \emph{Applied Soft Computing}, \textbf{45} (2016), 98--107.

\bibitem{Dua2019}
\newblock D.~Dua and C.~Graff,
\newblock {UCI} machine learning repository, 2017,
\newblock \urlprefix\url{http://archive.ics.uci.edu/ml}.

\bibitem{Fefferman2016}
\newblock C.~Fefferman, S.~Mitter and H.~Narayanan,
\newblock Testing the manifold hypothesis,
\newblock \emph{Journal of the American Mathematical Society}, \textbf{29}
  (2016), 983--1049.

\bibitem{Fischer2003}
\newblock B.~Fischer and J.~Buhmann,
\newblock Path-based clustering for grouping of smooth curves and texture
  segmentation,
\newblock \emph{IEEE Transactions on Pattern Analysis and Machine
  Intelligence}, \textbf{25} (2003), 513--518.

\bibitem{Har2016}
\newblock S.~Har-Peled,
\newblock Computing the k nearest-neighbors for all vertices via dijkstra,
\newblock \emph{arXiv preprint arXiv:1607.07818}.

\bibitem{Ho2003}
\newblock J.~Ho, M.-H. Yang, J.~Lim, K.-C. Lee and D.~Kriegman,
\newblock Clustering appearances of objects under varying illumination
  conditions,
\newblock in \emph{CVPR (1)}, 2003,
\newblock 11--18.

\bibitem{Howard2001}
\newblock C.~D. Howard and C.~Newman,
\newblock Geodesics and spanning trees for euclidean first-passage percolation,
\newblock \emph{Annals of Probability}, 577--623.

\bibitem{Hwang2016}
\newblock S.~Hwang, S.~Damelin and A.~Hero~{III},
\newblock Shortest path through random points,
\newblock \emph{The Annals of Applied Probability}, \textbf{26} (2016),
  2791--2823.

\bibitem{Jacobs2018}
\newblock M.~Jacobs, E.~Merkurjev and S.~Esedoḡlu,
\newblock Auction dynamics: A volume constrained mbo scheme,
\newblock \emph{Journal of Computational Physics}, \textbf{354} (2018),
  288--310.

\bibitem{Little2017}
\newblock A.~Little, M.~Maggioni and J.~Murphy,
\newblock Path-based spectral clustering: Guarantees, robustness to outliers,
  and fast algorithms,
\newblock \emph{arXiv preprint arXiv:1712.06206}.

\bibitem{Moscovich2017}
\newblock A.~Moscovich, A.~Jaffe and B.~Nadler,
\newblock Minimax-optimal semi-supervised regression on unknown manifolds,
\newblock in \emph{Artificial Intelligence and Statistics}, 2017,
\newblock 933--942.

\bibitem{Nene1996}
\newblock S.~Nene, S.~Nayar, H.~Murase et~al.,
\newblock Columbia object image library (coil-20).

\bibitem{Ng2002}
\newblock A.~Ng, M.~Jordan and Y.~Weiss,
\newblock On spectral clustering: Analysis and an algorithm,
\newblock in \emph{Advances in neural information processing systems}, 2002,
\newblock 849--856.

\bibitem{Orlitsky2005}
\newblock A.~Orlitsky and Sajama,
\newblock Estimating and computing density based distance metrics,
\newblock in \emph{Proceedings of the 22nd international conference on Machine
  learning},
\newblock ACM, 2005,
\newblock 760--767.

\bibitem{Tenenbaum2000}
\newblock J.~Tenenbaum, V.~De~Silva and J.~Langford,
\newblock A global geometric framework for nonlinear dimensionality reduction,
\newblock \emph{science}, \textbf{290} (2000), 2319--2323.

\bibitem{Vincent2003}
\newblock P.~Vincent and Y.~Bengio,
\newblock Density-sensitive metrics and kernels,
\newblock in \emph{Snowbird Learning Workshop}, 2003.

\bibitem{Yin2018}
\newblock K.~Yin and X.-C. Tai,
\newblock An effective region force for some variational models for learning
  and clustering,
\newblock \emph{Journal of Scientific Computing}, \textbf{74} (2018), 175--196.

\bibitem{Zelnik2005}
\newblock L.~Zelnik-Manor and P.~Perona,
\newblock Self-tuning spectral clustering,
\newblock in \emph{Advances in neural information processing systems}, 2005,
\newblock 1601--1608.

\end{thebibliography}
\end{document}